\documentclass[accepted,onefignum,onetabnum]{siamart190516}

\usepackage{lipsum}
\usepackage{amsfonts}
\usepackage{graphicx}
\usepackage{epstopdf}
\usepackage{algorithmic}
\ifpdf
  \DeclareGraphicsExtensions{.eps,.pdf,.png,.jpg}
\else
  \DeclareGraphicsExtensions{.eps}
\fi

\usepackage{url}

\usepackage{breakurl}
\usepackage{amsmath}
\usepackage{amssymb}
\usepackage{bbm}
\usepackage{mathrsfs}
\usepackage{algorithm}
\usepackage{algorithmic}
\usepackage{tikz-cd}

\newcommand{\eps}{\varepsilon}

\newcommand{\mc}[1]{{\mathcal{#1}}}

\newcommand{\bb}[1]{{\mathbb{#1}}}

\renewcommand{\tt}[1]{{\texttt{#1}}}

\newcommand{\Z}{\mathbb{Z}}

\newcommand{\R}{\mathbb{R}}

\newcommand{\C}{\mathbb{C}}

\renewcommand{\bf}{\textbf}

\newsiamremark{remark}{Remark}
\newsiamremark{hypothesis}{Hypothesis}
\crefname{hypothesis}{Hypothesis}{Hypotheses}
\newsiamthm{claim}{Claim}

\headers{Wassmap: Wasserstein Isometric Mapping}{K. Hamm, N. Henscheid, and S. Kang}

\title{Wassmap: Wasserstein Isometric Mapping for Image Manifold Learning}

\author{Keaton Hamm\thanks{Department of Mathematics, University of Texas at Arlington
  (\email{keaton.hamm@uta.edu})} \and Nick Henscheid\thanks{Department of Medical Imaging, University of Arizona 
  (\email{nph@email.arizona.edu}).} \and Shujie Kang\thanks{Department of Mathematics, University of Texas at Arlington (\email{shujie.kang@uta.edu}). \funding{KH was sponsored in part by the Army Research Office under grant number W911NF-20-1-0076. The views and conclusions contained in this document are those of the authors and should not be interpreted as representing the official policies, either expressed or implied, of the Army Research Office or the U.S. Government. The U.S. Government is authorized to reproduce and distribute reprints for Government purposes notwithstanding any copyright notation herein. 
  NH was supported by the NIBIB grant R01EB000803.}} 
}

\usepackage{amsopn}

\newcommand{\sspan}{\textnormal{span}}

\def\edit{\textcolor{black}}

\begin{document}

\maketitle

\begin{abstract}

In this paper, we propose Wasserstein Isometric Mapping (Wassmap), a nonlinear dimensionality reduction technique that provides solutions to some drawbacks in existing global nonlinear dimensionality reduction algorithms in imaging applications. Wassmap represents images via probability measures in Wasserstein space, then uses pairwise Wasserstein distances between the associated measures to produce a low-dimensional, approximately isometric embedding. We show that the algorithm is able to exactly recover parameters of some image manifolds including those generated by translations or dilations of a fixed generating measure. Additionally, we show that a discrete version of the algorithm retrieves parameters from manifolds generated from discrete measures by providing a theoretical bridge to transfer recovery results from functional data to discrete data. Testing of the proposed algorithms on various image data manifolds show that Wassmap yields good embeddings compared with other global and local techniques.

\end{abstract}



\section{Introduction}

One of the fundamental observations of data science is that high-dimensional data often exhibits low-dimensional structure.  Detecting and utilizing structures such as sparsity, union of subspaces, or low-dimensional manifolds has been the driving force of innovation and success for many modern algorithms pertaining to image and video processing, clustering, and pattern recognition, and has led to better understanding of the success of neural network classifiers and other machine learning models. In particular, a common assumption in machine learning is the \textit{manifold hypothesis} \cite{chen2012nonlinear,fefferman2016testing,jones2008manifold,lee2007nonlinear}, which is that data lies on or near a low-dimensional embedded manifold in the high-dimensional ambient space.

Myriad manifold learning algorithms have been proposed for elucidating the structure of these manifolds by embedding the data into a significantly lower-dimensional space, e.g.,  \cite{belkin2003laplacian,coifman2006diffusion,donoho2003hessian,hinton2006reducing,maaten2008visualizing,tenenbaum2000global} among many others.  Such methods have been applied to data as diverse as stock prices \cite{huang2017nonlinear}, medical images \cite{wolz2012nonlinear} and single-cell sequencing data \cite{becht2019dimensionality}.

\subsection{Challenges in image manifold learning}

Many applications result in Euclidean data in $\R^n$ or $\C^n$ for large $n$. However, in imaging applications in which data is obtained through photography, video recording, hyperspectral imaging, MRI, or related methods, the resulting Euclidean vectors, matrices, or tensors are better modeled as functional data, since images correspond to objects that are naturally thought of as \textit{prima facie} infinite-dimensional. That is, one obtains
\begin{equation}\label{EQN:Imaging} x = \mathcal{H}[f]+\eta,\end{equation}
where $x$ is the (discrete) image, $\mathcal{H}:X\to\R^n$ is an imaging (or discretization) operator mapping a Banach space $X$ (often $L_p(\R^m)$ for some $p$ and $m$, or more commonly $L_p(\Omega)$ for some compact $\Omega\subset\R^m$) and $\eta$ is some noise (often treated as stochastic and is based on the imaging operator as well as other external factors). 

The noise $\eta$ can come from multiple sources including background noise, e.g., randomly occurring features such as non-diseased tissue that are not of primary interest for the task \cite{rolland1992effect}; such noise typically has much higher intrinsic dimension than the signal of interest. Image data can also be corrupted by electronic and quantum noise, which is particularly prevalent in scientific and clinical medical imaging where, for instance, radiation dose may prevent the usage of strong light sources \cite{barrett2015task}.

Many dimensionality reduction methods operate \edit{by forming an $\eps$--neighborhood of $k$--nearest neighbor graph over sample points and embedding this graph (or one derived from it) into some much smaller dimensional space (or sometimes an infinite dimensional space).} 
Examples which use variations of this procedure are Isomap \cite{tenenbaum2000global}, Local Linear Embedding \cite{roweis2000nonlinear},  Laplacian Eigenmaps \cite{belkin2003laplacian}, UMAP \cite{mcinnes2018umap}, and Diffusion Maps \cite{coifman2006diffusion}.  Each step above has been an avenue of substantial research. While these methods have enjoyed great success in many areas, there are a few drawbacks, especially for imaging applications.  First, the graph formation step is typically done in a heuristic fashion and can be problematic in that it is very sensitive to parameter tuning, i.e., choosing $\eps$ or $k$ and how to weight the edges.


Second, the most common framework above assumes that $\{x_i\}$ comes from a Riemannian manifold embedded in Euclidean space. Under this assumption, variants of the above procedure are designed so that (hopefully) the graph-theoretic geodesics closely approximate the manifold geodesics between data points. Bernstein et al.~\cite{bernstein2000graph} prove that if the sampling density of the points $\{x_i\}$ is sufficiently small with respect to the minimum radius of curvature of the ambient manifold and a prescribed tolerance, the graph geodesics of an $\eps$--neighborhood graph can approximate the manifold geodesics between all pairs $(x_i,x_j)$ within the prescribed tolerance. Their results show that success typically requires dense sampling of the manifold, which is often unrealistic in image applications due to sparsity of sampling. 
Additionally, images of the same objects can have different dimensions ($n$) in the imaging domain under different imaging systems, which may lead to quite different results in the dimensionality reduction procedure. Many algorithms will downsample images to alleviate this issue, but this can lead to information loss and is not necessary in our proposed framework. 

Finally, such models tacitly assume that Euclidean distances between data vectors are semantically meaningful.  However, this assumption may be invalid in many applications. Indeed, small variation in pixel intensities results in large Euclidean distances, but the images may be semantically the same.  For instance, in object recognition, one would expect a model to understand that two images of a car are the same object even if the car is translated in the frame of one of the images. These two images can have large Euclidean distance, even though they are semantically identical.

\subsection{Functional image manifolds in Wasserstein space}

We propose the following paradigm shift from the previous discussion. In contrast to many imaging techniques which assume that imaged data is on a manifold without reference to the function space underlying them, we assume a \textit{functional manifold hypothesis} that $\{x_i\}\subset\R^n$ is obtained from imaging a functional manifold $\mathcal{M}$.  A natural question arises: what function space naturally represents image data?  Our analysis below assumes that images correspond to probability measures with finite $p$-th moment; i.e., that $x = \mathcal{H}(\mu)$ as in \eqref{EQN:Imaging} where $\mu\in\edit{\mathbb{W}_p}(\R^m)$, the $p$--Wasserstein space of probability measures with finite $p$-th moment $M_p(\mu):=\int_{\R^m}|x|^pd\mu(x)<\infty$.

The $p$--Wasserstein space is equipped with the Wasserstein metric arising from Optimal Transport Theory \cite{villani2008optimal}.  Given two measures $\mu,\nu\in\edit{\mathbb{W}_p}(\R^m)$, define the set of couplings $\Gamma(\mu,\nu):=\{\gamma\in\mathcal{P}(\R^{2m}):\pi_1\gamma = \mu, \pi_2\gamma = \nu\}$ where $\mathcal{P}(\R^{2m})$ is the set of all probability measures on $\R^{2m}$, $\pi_1$ is the projection onto the first $m$ coordinates, and $\pi_2$ onto the last $m$ coordinates. Thus $\Gamma(\mu,\nu)$ is the set of all joint probability measures on $\R^{2m}$ whose marginals are $\mu$ and $\nu$ \cite{santambrogio2015optimal}.  Then we may define
\begin{equation}\label{EQN:Wasserstein}
\edit{W_p(\mu,\nu) = \inf_{\gamma\in\Gamma(\mu,\nu)}\left(\int_{\R^{2m}}|x-y|^pd\gamma(x,y)\right)^\frac1p.}
\end{equation}
Our initial assumption will be that measures are absolutely continuous; however, we will provide a bridge to transfer results from these to arbitrary measures in \cref{SEC:DiscreteWassmap}. 

Treating images as probability measures and considering their $p$--Wasserstein distances mitigates issues of geodesic blowup inherent in assuming $L_2$ as the ambient function space (see \cite{donoho2005image}) as $\mathbb{W}_p$ is a \textit{length space}, meaning it contains all geodesics, which are necessarily finite \cite{santambrogio2015optimal,sturm2006geometry}.  Additionally, the Wasserstein distance contains more semantic meaning in that the Wasserstein distance between images captures how much energy is required to morph one image into another, and the displacement interpolant (the line from one measure to another in $\mathbb{W}_p$) provides a more natural nonlinear path between measures (see \cite{kolouri2017optimal}, for instance). Our initial theoretical and experimental results indicate that Wasserstein distances have significant advantages over other choices in terms of recovering image manifold parametrizations and providing good low-dimensional embeddings of image manifolds (\cref{SEC:Experiments}).

Additionally, use of Wasserstein distance and optimal transport theory provides one with a significant and powerful theoretical framework due to the substantial work on optimal transport related to PDEs and other fields (e.g., \cite{peyre2019computational,santambrogio2015optimal,villani2003topics,villani2008optimal}). In addition to the bulk of theory developed, the use of optimal transport in the past few years has yielded a plethora of advances to the state-of-the-art in many subfields of Machine Learning (ML).  For example, use of Wasserstein distances in training of Generative Adversarial Networks (GANs) leads to substantial improvement and stability of such networks \cite{WassersteinGAN,WassersteinProximal}, and in image processing, use of optimal transport ideas has enabled linearization of nonlinear classification problems \cite{kolouri2016radon,kolouri2017optimal}.

Note that the Wasserstein manifold assumption is general, and subsumes a setting which is natural in imaging applications. In many cases, we may readily assume that objects being imaged are compactly supported, nonnegative, and integrable (e.g., we may consider a car to be an element of $L_1^{\geq0}(\Omega)$ for a compact set $\Omega\subset\R^3$, where the function value at a given point is the density of the car in that location).  Thus, one could consider the data $\{x_i\}\subset\R^n$ to be obtained from imaging a functional manifold $\mathcal{M}\subset L_1^{\geq0}(\Omega)$ for some compact set $\Omega\subset\R^m$. Assuming the images have unit $L_1$--norm implies that we may view this manifold as a subset of $\mathcal{P}(\Omega)$, the set of probability measures supported on $\Omega$ via the mapping $f_i\mapsto \mu_i$ such that $d\mu_i=f_idx$ (with $dx$ being the Lebesgue measure on $\R^m$).  \edit{These measures have finite $p$-th moment since
\[\int_\Omega |x|^pd\mu_i(x) \leq \max_{x\in\Omega}|x|^p \int_\Omega d\mu_i(x) <\infty.\]}
Therefore, each $\mu_i$ is an element of the $p$--Wasserstein space $\mathbb{W}_p(\R^m)$.

\subsection{Main results}

We briefly summarize our main results here. For some definitions and background on optimal transport and Wasserstein distance, see \cref{SEC:OTP}. We propose a variant of Isomap called Wassmap which uses Wasserstein distances instead of Euclidean distances. \edit{For theory in this work, we treat the case when the graph geodesic computation is excluded, which is a Wasserstein distance based Multidimensional Scaling algorithm; however in experiments we illustrate the behavior of our algorithm with and without the graph geodesic step.} This algorithm and the assumption that images correspond to elements of $\mathbb{W}_p$ are used to explore settings in which image manifold parametrizations can be exactly recovered up to rigid transformation by our algorithm. \edit{The main results of this paper are the following (here we state things informally, and rigorous theorems follow in later sections). 
\begin{itemize}
    \item Given discrete samples from a smooth submanifold of $\mathbb{W}_p$ that is isometric to Euclidean space, the Functional Wassmap Algorithm (\cref{ALG:Wassmap}) yields an isometric embedding and recovers the parameter set governing the manifold up to rigid transformation.
    \item For manifolds generated by translates or dilations of a fixed absolutely continuous measure, Functional Wassmap (\cref{ALG:Wassmap}) recovers the translation set or a scaled version of the dilation set up to rigid transformation. This result holds for general $p$ for translations, $p=2$ for anisotropic dilations, and general $p$ for isotropic dilations.
    \item For submanifolds of $\mathbb{W}_p$ generated by ``nice" parametrized sets of diffeomorphisms acting on a generating measure, the Wasserstein distances between pairs of measures are the same whether the measure is discrete or absolutely continuous.
\end{itemize}
}

The final result provides a bridge to transfer results for absolutely continuous measures to discrete measures which arise in practice in imaging applications, e.g., as obtained via \eqref{EQN:Imaging}. More specifically, it implies that in certain cases, if Functional Wassmap (\cref{ALG:Wassmap}) recovers an image manifold parametrization for absolutely continuous measures, then it recovers the parametrization for arbitrary measures, and Discrete Wassmap (\cref{ALG:DiscreteWassmap}) recovers the parametrization for discrete measures.

After discussing prior art in the next subsection, the rest of the paper is organized as follows: \cref{SEC:OTP} describes in brief the background of Wasserstein distances and other details from optimal transport theory, \cref{SEC:Theory} describes the Functional Wassmap algorithm and contains the main results and proofs related to it, \cref{SEC:DiscreteWassmap} describes the Discrete Wassmap algorithm and the theorem transferring Wasserstein computations from the continuous to the discrete measure case. \Cref{SEC:Experiments} contains experiments, \cref{SEC:Computation} discusses computational aspects of the algorithm, and we end with a brief conclusion section.

\subsection{Prior art}

Most nonlinear dimensionality reduction methods assume data on or near a low-dimensional manifold in Euclidean space and utilize Euclidean distances between points to estimate manifold geodesics. Isomap, described by Tenenbaum et al.~\cite{tenenbaum2000global} is one of the most classical of these algorithms and is the inspiration of this work. Bernstein et al.~\cite{bernstein2000graph} showed that dense sampling is required to well-approximate geodesics in the Isomap procedure, still under the assumption of Euclidean manifolds. Zha and Zhang \cite{zha2003isometric} proposed continuum Isomap, assuming continuous sampling of the manifold, and utilizing an integral operator formulation of Isomap and Multi-dimensional Scaling (MDS) \cite{mardia1979multivariate}. Continuum Isomap therein illuminates the theory of classical Isomap maintaining the Euclidean manifold assumption, but is not a practical algorithm.  

Donoho and Grimes \cite{donoho2005image} utilized the functional manifold hypothesis that data lives in a submanifold of $L_2(\R^m)$ as a theoretical tool to study the performance of Isomap.  Due to the fact that geodesics formed by the metric induced by the $L_2$ metric can blow up, e.g., in the case of translates of indicator functions, the authors require convolution of the input measures by Gaussians.  They also utilize normalization with respect to a reference geodesic, which our framework does not require.

The works of Kolouri et al.~\cite{kolouri2019generalized,kolouri2016radon,kolouri2017optimal} and others \cite{aldroubi2021partitioning,khurana2022supervised,moosmuller2020linear} consider absolutely continuous measures, but instead of working with the Wasserstein distances directly, they work with $L_2$ distances between the optimal transport maps.  This approach can speed up computations \cite{khurana2022supervised} compared to our approach, but the theory does not transfer to arbitrary measures as is done here. As with the Donoho and Grimes framework, these works also require a reference image to define the transport distance, whereas our method of utilizing Wasserstein distances directly avoids this. Additionally, exact recovery results for image manifold parametrizations appear to be easier to obtain with our framework than these approaches. 

Recently, Kileel et al.~\cite{kileel2021manifold} study the problem of manifold learning with arbitrary norms. Their assumption remains that the data manifold is embedded in Euclidean space, but they construct an analogue of the graph Laplacian which utilizes an arbitrary norm as opposed to the standard Euclidean norm. 
In similarity with our work, Kileel et al.~are motivated by the fact that Euclidean distances may lack semantic meaning or may lead to inflated computational load compared with other norms. Additionally, in their experiments they employ an approximation of the $W_1$ distance using wavelet expansions for sparse representations of image data. This is computationally faster than $W_p$ approximations; however, the results presented herein are primarily aimed at understanding exact recovery of image manifold parametrizations up to rigid motion, \edit{and uniqueness of transport maps holds in $p\in(1,\infty)$ but not in the case $p=1$, so we focus on this range of $p$; though our method could be applied to the $p=1$ case, but we leave this to future work.} 

\edit{Wang et al. \cite{wang2010optimal} utilized opimal transport metrics to compare images of nuclear chromatin. Part of the method they applied is the same as our \cref{ALG:DiscreteWassmap} without the optional graph geodesic step. They used $W_2$ distance to quantitatively measure the difference of nuclei, then combined MDS and Fisher Discriminant Analysis to study the distributions of nuclei. They apply the method for classification, which is different from our end goal of retrieving parameters generating the image manifolds.}

\section{Basics of Wasserstein Distances and Optimal Transport}\label{SEC:OTP}


Given a measure space $X$, we denote the space of all finite measures on $X$ by $\mc{M}(X)$. Given a measure space $Y$, and a continuous map $T:X\to Y$, the pushforward of a measure $\mu\in\mc{M}(X)$ via the map $T$, denoted $T_\#\mu\in\mc{M}(Y)$, is the measure which satisfies \[T_\#\mu(E) = \mu(T^{-1}(E)),\quad \textnormal{for all measurable } E\subset Y. \]

Suppose $X$ and $Y$ are measure spaces, $c:X\times Y\to\R_+$ is a cost function, and $\mu\in\mc{M}(X)$ and $\nu\in\mc{M}(Y)$; then the Monge Problem is to find the \textit{optimal transport map} $T:X\to Y$ which minimizes
\begin{equation}\label{EQN:MP}\tag{MP}
\min_T 
\left\{\int_X c(x,T(x))d\mu(x): T_\#\mu = \nu \right\}.
\end{equation}

The Monge problem does not always admit a solution, even in seemingly innocuous cases such as discrete measures. To get around this difficulty, Kantorovich proposed the following relaxation, which we call the \textit{Kantorovich Problem}: 

\begin{equation}\label{EQN:KP}\tag{KP}
\min_\pi \left\{ \int_{X\times Y} c(x,y)d\pi(x,y):\pi\in \Pi(\mu,\nu)\right\}
\end{equation}
where $\Pi(\mu,\nu)$ is the class of transport plans, or couplings:
	\begin{equation*}
	\Pi(\mu,\nu) = \{\edit{\gamma}\in\mc{P}(X\times Y): (\pi_x)_\sharp\gamma = \mu, (\pi_y)_\sharp\gamma = \nu\}.
	\end{equation*}
Here, $\pi_x, \pi_y$ are the projections of $X\times Y$ onto $X$ and $Y$ (i.e., the marginals on $X$ and $Y$), respectively.

Of use to us will also be the \textit{Dual Problem} to the Kantorovich Problem:

\begin{equation}\label{EQN:DP}\tag{DP}
\sup \left\{\int_{X}\phi d\mu + \int_Y\psi d\nu:\phi\in L_1(\mu), \psi\in L_1(\nu), \phi(x)+\psi(y)\leq c(x,y) \right\}.    
\end{equation}

Finally, intimately tied to all of these problems is the \textit{Wasserstein Distance.}  Here, we specialize to the concrete case $X=Y=\R^n$ and utilize the $\ell_2$--norm (quadratic) cost function, i.e., $c(x,y):=|x-y|^2$ (here and throughout, $|\cdot|$ denotes the Euclidean norm on $\R^m$ where $m$ may be determined from context).  For $\mu,\nu\in\mathcal{P}(\R^m)$ with finite $2$--nd moment, the $2$--Wasserstein distance is defined by \eqref{EQN:Wasserstein}.


Evidently, $W_2(\mu,\nu)^2 = (\min$--KP), but in this setting much more is true. The following is a combination of several results in \cite[Chapter 1]{santambrogio2015optimal} and Brenier's Theorem \cite{brenier1991polar} (see also \cite{peyre2019computational}). 

\begin{theorem}\label{THM:ProbEquiv}
Let $c(x,y)=|x-y|^2$. Suppose $\mu,\nu\in\mathbb{W}_2(\R^m)$, and at least one of which is absolutely continuous.  Then there exists an optimal transport map $T$ from $\mu$ to $\nu$ and a unique optimal transport plan $\pi\in\Pi(\mu,\nu)$.  Additionally,
\[(\min\textnormal{--MP}) = (\min\textnormal{--KP}) = (\max\textnormal{--DP}) = W_2(\mu,\nu)^2.\]
\end{theorem}

\edit{
For $p\in(1,\infty)$, the optimal transport map is unique and can be characterized by the following theorem of Gangbo and McCann \cite{gangbo1996geometry}. A function $\psi:\R^d\rightarrow \R\cup\{-\infty\}$ is called \textit{c-concave} if there exist a set $A\subset \R^d\times \R$ such that
\[\psi(x)=\inf\limits_{(y,\lambda)\in A}c(x,y)+\lambda.\]
\begin{theorem}[\cite{gangbo1996geometry}]\label{THM:GangboMcCann}
    Let $c(x,y) = |x-y|^p=: h(x-y)$, where $p\in(1,\infty)$. Suppose $\mu,\nu \in \mathbb{W}_p(\R^m)$. If $\mu$ is absolutely continuous with respect to Lebesgue measure then a map $T$, which solves the Monge problem, pushing $\mu$ forward to $\nu$ is uniquely determined $\mu$-almost everywhere by the requirement that it be of the form $T(x)=x-\nabla h^{-1}\big(\nabla\psi(x)\big)$ for some c-concave $\psi$ on $\R^d$.
\end{theorem}
}

\section{Functional Wassmap: Algorithm and Theory}\label{SEC:Theory}

In this section, we consider the problem of when an image manifold treated as a submanifold of the Wasserstein space is isometric to Euclidean space. \edit{We state the Functional Wassmap algorithm below in full generality, but for treatment of theoretical guarantees we will restrict the class of manifolds we consider.  First, we consider the case when the geodesics on the manifold $\mc{M}\subset\mathbb{W}_p(\R^m)$ are given by the $W_p$ distance between measures. Below, Step 4 is the optional step of forming a graph whose nodes are the measures and there is some rule for determining neighborhoods of nodes and setting edge weights (for instance, $\eps$--neighborhood or $k$--nearest neighbors). We use APSP to stand for All-pairs shortest path, i.e., computing graph-theoretic distances between all nodes (for example via Dijkstra's algorithm). Curved manifolds in Wasserstein space will require this graph geodesic step, whereas classical results regarding Multidimensional Scaling imply that flat manifolds do not require this step.} For our theoretical results, we restrict to cases when the metric space $(\mc{M},\edit{W_p})$ is isometric up to a constant to a subset of Euclidean space $(\Theta,|\cdot|_{\R^d})$, and suppress the graph geodesic step in \cref{ALG:Wassmap}.

\begin{algorithm}[h!]
 \caption{Functional Wasserstein Isometric Mapping (Functional Wassmap)}\label{ALG:Wassmap}
\begin{algorithmic}[1]

\STATE \textbf{Input: }{Probability measures $\{\mu_i\}_{i=1}^N\subset \edit{\mathbb{W}_p}(\R^m)$; embedding dimension $d$
}

\STATE \textbf{Output: }{Low-dimensional embedding points $\{z_i\}_{i=1}^N\subset\R^d$}

\STATE{Compute pairwise Wasserstein distance matrix $W_{ij} = \edit{W_p}^2(\mu_i,\mu_j)$}
\STATE{(Optional) Form neighborhood graph $G$ using $W$, and set $W=\textnormal{APSP}(G)$}
\STATE{ $B = -\frac12 HWH$, where  ($H=I-\frac{1}{N}\mathbbm{1}_N)$}

\STATE (Truncated eigendecomposition):{ $B_d=V_d\Lambda_d V_d^T$}

\STATE{$z_i = (V_d\Lambda_d^\frac{1}{2})(i,:),$ for $i=1,\dots,N$}

\STATE \textbf{Return:}{$\{z_i\}_{i=1}^N$}
\end{algorithmic}
\end{algorithm}

\subsection{Multidimensional scaling}

Steps 5--7 of \cref{ALG:Wassmap} are the \textit{classical Multi-dimensional Scaling} Algorithm, or MDS. An important result for MDS is the following.  

\begin{definition}
A matrix $D\in\R^{N\times N}$ is a \textit{distance matrix} provided $D=D^T$, $D_{ii}=0$ for all $i$, and $D_{ij}\geq0$ for all $i\neq j$.

A distance matrix is \textit{Euclidean} provided there exists a point set $\{z_i\}_{i=1}^N\subset\R^d$ for some $d$ such that $D_{ij} = |z_i-z_j|^2.$
\end{definition}

\begin{theorem}[\cite{young1938discussion}]\label{THM:MDS}
Let $D$ be a distance matrix, \edit{$B=-\frac12 HDH$}, and $V_d,$ and $\Lambda_d$ be as in \cref{ALG:Wassmap}.  $D$ is Euclidean if and only if is symmetric positive semidefinite.  Moreover, if $D$ is Euclidean, then the points $\{z_1,\dots,z_N\}$ are unique up to rigid transformation and are given by $(V_d\Lambda_d^\frac12)(i,:),$  $i=1,\dots,N.$
\end{theorem}

\begin{corollary}\label{COR:WassmapRecovery}
Let \edit{$p\in(1,\infty)$} and $\Theta\subset\R^d$ be a parameter set that generates a smooth submanifold $\mc{M}(\Theta)\subset\edit{\mathbb{W}_p}(\R^m)$ such that $(\mc{M},\edit{W_p})$ is isometric up to a constant to $(\Theta,|\cdot|_{\R^d})$.  If $\{\theta_i\}_{i=1}^N\subset\Theta$, and $\{\mu_{\theta_i}\}_{i=1}^N\subset\mc{M}$ are the corresponding measures on the manifold, then the Functional Wassmap Algorithm (\cref{ALG:Wassmap}) with embedding dimension $d$ recovers $\{\theta_i\}$ up to rigid transformation and global scaling.
\end{corollary}

\begin{proof}
    The isometry condition implies existence of a global constant $c>0$ such that $\edit{W_p}(\mu_{\theta_i},\mu_{\theta_j}) = c|\theta_i-\theta_j|$ for all $i,j$. Hence the matrix $W$ in \cref{ALG:Wassmap} is a Euclidean distance matrix with point configuration $\{c\theta_i\}_{i=1}^N\subset\R^d$, and uniqueness up to rigid transformation is given by \cref{THM:MDS}.
\end{proof}



In subsequent results, we will compute the global scaling factor $c$ for some image manifolds, in which case we obtain recovery of $\{c\theta_i\}$ up to rigid transformation.

\subsection{Comparison to other techniques}









Donoho and Grimes \cite{donoho2005image} developed a theoretical framework for understanding the behavior of Isomap on image manifolds. They studied whether a normalized version of geodesic distance is equivalent to the Euclidean distance, in which case Isomap recovers the underlying parametrization of image manifolds. They show several positive cases including translation, pivoting and morphing boundaries of black objects on white backgrounds. They also show that Isomap may fail when the parameter space is not convex or the image manifold is not flat (for example, the dilation manifold of rectangles or ellipses). 

In comparison, Wassmap does not require normalization when computing distances between images. For translation manifolds, Wassmap retreives the underlying parameters without requiring the parameter space to be convex. Wassmap also recovers translation and dilation manifolds generated by a base measure which has nonsmooth pdf, like the indicator function of a domain, whereas Isomap fails in this case due to geodesic blowup.

\subsection{Translation manifolds}

Given a fixed generating measure $\mu_0\in\edit{\mathbb{W}_p}(\R^m)$ and translation set $\Theta\subset\R^m$, define

\begin{equation}
    \mathcal{M}^{\textnormal{trans}}(\mu_0,\Theta):=\{\mu_0(\cdot-\theta):\theta\in\Theta\}.\label{eqn:trans_manifold_def}
\end{equation}
This simple translation manifold satisfies $\dim(\mathcal{M}) = \dim(\sspan(\Theta))$. We show the following:

\begin{theorem}\label{THM:WassmapTranslation}
Let \edit{$p\in(1,\infty)$ and} $\mu_0\in\edit{\mathbb{W}_p}(\R^m)$ be absolutely continuous. Given $\{\theta_i\}_{i=1}^N\subset\R^m$ and corresponding measures $\{\mu_{\theta_i}\}_{i=1}^N\subset\mathcal{M}^{\textnormal{trans}}(\mu_0,\Theta)$, the Functional Wassmap algorithm (\cref{ALG:Wassmap}) with embedding dimension $m$ recovers $\{\theta_i\}_{i=1}^N$ up to rigid transformation.
\end{theorem}

The crux of the proof of this theorem is the following lemma. This lemma is known \cite[Remark 2.19]{peyre2019computational}, but for completeness we present the full proof.

\begin{lemma}\label{LEM:Translation}
Let \edit{$p\in(1,\infty)$ and}  $\mu_0\in\edit{\mathbb{W}_p}(\R^m)$ be absolutely continuous, and $\theta,\theta'\in\R^m$.  Then,
\[\edit{W_p}(\mu_0(\cdot-\theta),\mu_0(\cdot-\theta')) = |\theta-\theta'|.\]
\end{lemma}

\edit{
\begin{proof}
    Note that $T(x) = x+\theta-\theta'$ is such that $T_\#(\mu_0(\cdot-\theta)) = \mu_0(\cdot-\theta')$. Let $\phi(x)=\langle p|\theta-\theta'|^{p-2}(\theta'-\theta), x \rangle$. Then $T$ and $\phi$ satisfy $T(x)=x-(\nabla h)^{-1}(\nabla \phi(x))$ and $\phi$ is c-concave. By \cref{THM:GangboMcCann}, $T$ is the optimal transport map, whence
    \begin{align*}
        W_p(\mu_0(\cdot-\theta),\mu_0(\cdot-\theta'))^p & = \int_{\R^m}|x-(x+\theta-\theta')|^pf(x-\theta)dx\\
        & = |\theta-\theta'|^p\int_{\R^m}f(x)dx\\
        & = |\theta-\theta'|^p.
    \end{align*}
\end{proof}
}

    

\begin{proof}[Proof of \cref{THM:WassmapTranslation}]
Combine \cref{LEM:Translation} and \cref{COR:WassmapRecovery}.
\end{proof}


\subsection{Dilation manifolds} \label{subsec:dilationthy}

Here we will consider dilation manifolds. \edit{We begin with the general case of anisotropic dilations along coordinate axes, but our results here are only valid for the case $p=2$.}  Given a dilation set $\Theta\subset\R^m_{+}$ ($\theta\in\Theta$ that has strictly positive entries $\vartheta_1,\dots,\vartheta_m$), we define the corresponding manifold with a fixed \edit{absolutely continuous} generator $\mu_0\in \mathbb{W}_2(\R^m)$ \edit{with density $f$} 
via \[\mc{M}^{\textnormal{dil}}(\mu_0,\Theta):=\left\{\det(D_\theta)\mu_0(D_\theta \cdot):\theta\in\Theta\right\},\]
\edit{where we use the slight abuse of notation and consider $d\mu_0(D_\theta\cdot)=f(D_\theta x)dx$} and the dilation matrix is defined by
\[D_\theta := \textnormal{diag}\left(\frac{1}{\vartheta_1},\dots,\frac{1}{\vartheta_m}\right).\]

\edit{Recall that $M_p(\mu):=\int_{\R^m}|x|^pd\mu(x)$ is the $p$-th moment of a measure $\mu\in\mathcal{P}(\R^m)$, and let $P_i\mu$ denote the $i$--th marginal of $\mu$ defined by \[P_i\mu(E) := \int_{\R\times\dots\R\times E\times\R\times\dots\R}d\mu(x), \qquad E\subset\R.\] In the results below, the $p$-th moment of the $i$-th marginal is thus $M_2(P_i\mu):=\int_{\R^m}|x_i|^2d\mu(x)$. This choice of notation rather than $\mu_i$ is to avoid confusion, as subscripts $0$ and $\theta$ will be used frequently in the sequel.}

\begin{theorem}\label{THM:WassmapDilationGeneral}
Let $\mu_0\in\mathbb{W}_2(\R^m)$ be absolutely continuous. Given $\{\theta_i\}_{i=1}^N\subset\R^m_+$, and corresponding measures $\{\mu_{\theta_i}\}_{i=1}^N\subset\mc{M}^{\textnormal{dil}}(\mu_0,\Theta),$ the Functional Wassmap Algorithm (\cref{ALG:Wassmap}) with embedding dimension $m$ recovers $\{S\theta_i\}_{i=1}^N\subset \mathbb{R}^m$ up to rigid transformation, where $S$ is the diagonal matrix
\[
S= \textnormal{diag}\big(M^{\frac{1}{2}}_2(P_1\mu_0),\cdots, M^{\frac{1}{2}}_2(P_m\mu_0)\big).
\]
\end{theorem}

This theorem can be derived from the following lemma.

\begin{lemma}\label{LEM:Dilation}
Let $\mu_0\in\mathbb{W}_2(\R^m)$ be absolutely continuous with density $f$. Let $\theta,\theta'\in\Theta\subset\R^m_+$, and let $\mu_\theta$ be defined by $d\mu_{\theta} = \det(D_\theta)f(D_\theta\cdot)dx$, and similarly for $d\mu_{\theta'}$. Then 
\[W_2(\mu_\theta,\mu_{\theta'})^2 = \sum_{i=1}^m|\vartheta_i-\vartheta'_i|^2\int_{\R^m}|x_i|^2d\mu_0.\]
\end{lemma}

\begin{proof}
The proof proceeds by using \eqref{EQN:MP} to find an upper bound for the Wasserstein distance in question, and \eqref{EQN:DP} to find a lower bound. These being the same, we use \cref{THM:ProbEquiv} to conclude the result. To show the upper bound, we use the fact that \eqref{EQN:MP} has a solution, and note that the map $T = D_{\theta'}^{-1}D_{\theta}$ satisfies $T_\#\mu_\theta = \mu_{\theta'}$.  Indeed, for any measurable $E\subset\R^m$, we have, via the substitution $x = D_{\theta'}^{-1}D_\theta y$,
\begin{align*}
    \mu_{\theta'}(E) & = \int_E \det(D_{\theta'})f(D_{\theta'}x)dx\\
    & = \int_{D_\theta^{-1}D_{\theta'}(E)} \det(D_\theta)f(D_\theta y)dy\\
    & = \mu_\theta(D_\theta^{-1}D_{\theta'}E)\\
    & = \mu_\theta(T^{-1}(E)).
\end{align*}
Hence, $T$ is the pushforward from $\mu_\theta$ to $\mu_{\theta'}$.  By \eqref{EQN:MP} and \cref{THM:ProbEquiv},

\begin{align*}
    W_2(\mu_\theta,\mu_{\theta'})^2 & \leq \int_{\R^m}|x-D_{\theta'}^{-1}D_\theta x|^2d\mu_\theta(x)\\
    & = \int_{\R^m}\sum_{i=1}^m \left|\left(1-\frac{\vartheta'_i}{\vartheta_i}\right)x_i\right|^2\det(D_\theta)f(D_\theta x)dx\\
    & = \sum_{i=1}^m \frac{1}{|\vartheta_i|^2}|\vartheta_i-\vartheta_i'|^2\int_{\R^m}|x_i|^2\det(D_\theta)f(D_\theta x)dx\\
    & = \sum_{i=1}^m |\vartheta_i-\vartheta_i'|^2\int_{\R^m}|x_i|^2f(x)dx\\
    & = \sum_{i=1}^m |\vartheta_i-\vartheta_i'|^2\int_{\R^m}|x_i|^2d\mu_0(x).
\end{align*}

The penultimate equality follows from substituting $x\mapsto D_\theta x$.

Now we use \eqref{EQN:DP} to find a lower bound for the Wasserstein distance by setting
\[\phi(x) = \sum_{i=1}^m\left(1-\frac{\vartheta_i'}{\vartheta_i}\right)x_i^2,\qquad \psi(y) = \sum_{i=1}^m\left(1-\frac{\vartheta_i}{\vartheta_i'}\right)y_i^2.\]
These are easily seen to be in $L_1(\mu_\theta)$ and $L_1(\mu_{\theta'})$, respectively. Additionally,
\\
\[|x-y|^2-\phi(x)- \psi(y) = \sum_{i=1}^m\left(\frac{\vartheta_i'}{\vartheta_i}x_i^2+\frac{\vartheta_i}{\vartheta_i'}y_i^2-2x_iy_i\right) = \sum_{i=1}^m\left(\sqrt{\frac{\vartheta_i'}{\vartheta_i}}x_i-\sqrt{\frac{\vartheta_i}{\vartheta_i'}}y_i\right)^2\geq0,\]
hence $\phi$ and $\psi$ are feasible solutions to \eqref{EQN:DP}. 

Finally, by \eqref{EQN:DP} and \cref{THM:ProbEquiv},

\begin{align*}
    W^2_2(\mu_\theta,\mu_{\theta'}) & \geq \int_{\R^m}\sum_{i=1}^m\left(1-\frac{\vartheta_i'}{\vartheta_i}\right)x_i^2\det(D_\theta)f(D_\theta x)dx\\ & \qquad\quad + \int_{\R^m}\sum_{i=1}^m\left(1-\frac{\vartheta_i}{\vartheta_i'}\right)y_i^2\det(D_{\theta'})f(D_{\theta'}y)dy\\
    & = \int_{\R^m}\sum_{i=1}^m(\vartheta_i^2-\vartheta_i\vartheta_i')x_i^2f(x)dx + \int_{\R^m}\sum_{i=1}^m((\vartheta_i')^2-\vartheta_i\vartheta_i')y_i^2f(y)dy\\
    & = \sum_{i=1}^m|\vartheta_i-\vartheta_i'|^2\int_{\R^m}|x_i|^2d\mu_0(x),
\end{align*}
and the lemma is proved.
\end{proof}

\begin{proof}[Proof of \cref{THM:WassmapDilationGeneral}]
\cref{LEM:Dilation} implies that $W_2(\mu_\theta,\mu_\theta') = |S\theta-S\theta'|$, so then the Wasserstein distance matrix arising from $\{\mu_{\theta_i}\}$ is a Euclidean distance matrix with point configuration $\{S\theta_i\}_{i=1}^N\subset\R^m$. The uniqueness of recovery of this set up to rigid transformation is given by \cref{THM:MDS}.
\end{proof}

Note that the matrix $S$ is determined by the generator $\mu_0$ of the manifold. Thus, in order to retrieve the parameters $\{\theta_i\}_{i=1}^N$, information on $\mu_0$ is required. In certain conditions, \cref{ALG:Wassmap} recovers $\{\theta_i\}_{i=1}^m$ up to a constant.

\edit{
\begin{remark}
     \cref{LEM:Dilation} may not generalize to any $p>1$ when the dilations are not isotropic (i.e., do not satisfy $\vartheta_i=\vartheta$ for all $i$). By \cref{THM:GangboMcCann}, $\phi$ and T are linked by $T(x)=x-(\nabla h)^{-1}(\nabla \phi(x))$, where $h(x-y)=|x-y|^p$, which means $\nabla h(x-T(x))$ should be the gradient of some function $\phi$. This is not true when the dilations are anisotropic.
\end{remark}
}

\begin{corollary}\label{COR:WassmapDilation}
    Suppose $\mu_0\in\mathbb{W}_2(\R^m)$ is such that $\int_{\R^m}|x_i|^2d\mu_0 = c^2$ for some constant $c>0$ and for all $i$. Given $\{\theta_i\}_{i=1}^N\subset\R^m_+$, and corresponding measures $\{\mu_{\theta_i}\}_{i=1}^N\subset\mc{M}^{\textnormal{dil}}(\mu_0,\Theta),$ the Functional Wassmap Algorithm (\cref{ALG:Wassmap}) with embedding dimension $m$ recovers $\{c\theta_i\}_{i=1}^N$ up to rigid transformation.
\end{corollary}

\begin{proof}
Combine \cref{LEM:Dilation} with \cref{COR:WassmapRecovery}.
\end{proof}

\begin{remark}
Note that if the dilations occur only along certain coordinates, i.e., $\Theta$ is supported on a $d$--dimensional coordinate plane for some $1\leq d<n$, then one can specify the embedding dimension in \cref{COR:WassmapDilation} to be $d$ rather than $m$.  In this case, one recovers the isometric projection of $\Theta$ into $\R^d$ which ignores the undilated coordinates.  For example, if $\Theta$ only has elements other than 1 in coordinates $\{i_1,\dots,i_d\}$, then Functional Wassmap will recover (up to rigid transformation) $P(\Theta)\subset\R^d$ where $P(\vartheta_1,\dots,\vartheta_m) := (\vartheta_{i_1},\dots,\vartheta_{i_d})$.
\end{remark}

\edit{For isotropic dilations, the result holds for arbitrary $p\in(1,\infty)$.}
\begin{corollary}[Isotropic Dilations]\label{THM:WassmapIsotropicDilation}
Suppose \edit{$p\in(1,\infty)$, $\mu_0\in\mathbb{W}_p(\R^m)$} is absolutely continuous and $\{\theta_i\}_{i=1}^N\subset\Theta\subset\{c(1,\dots,1):c\in\R_+\}\subset\R^m$. Given the corresponding measures $\{\mu_{\theta_i}\}_{i=1}^N\subset\mc{M}^{\textnormal{dil}}(\mu_0,\Theta)$, the Functional Wassmap Algorithm (\cref{ALG:Wassmap}) with embedding dimension $m$ recovers \edit{$\{(\frac{M_p(\mu_0)}{m})^\frac1p\theta_i\}_{i=1}^N$} up to rigid transformation.
\end{corollary}


\edit{
\begin{proof}
    Suppose that $\theta = (c,\dots,c)$ and likewise $\theta' = (c',\dots,c')$.
    According to \cref{THM:GangboMcCann}, $T$ is the unique optimal transport map if $T(x)=x-(\nabla h)^{-1}(\nabla \phi(x))$ for some c-concave function $\phi$. The equation holds for $T=D_{\theta'}^{-1}D_{\theta}$ and $\phi(x)=\big(1-\frac{c}{c'}\big)^{p-1}|x|^p$. Then we have
    \[W_p(\mu_{\theta},\mu_{\theta'})^p = |c-c'|^p\int_{\R^m}\Bigg(\sum_{i=1}^m|x_i|^2\Bigg)^{\frac{p}{2}}d\mu_0 = |c-c'|^pM_p(\mu_0) = |\theta-\theta'|^p\frac{M_p(\mu_0)}{m},\] 
\end{proof}
}

Note that $\{(\frac{M_p(\mu_0)}{m})^\frac1p\theta_i\}$ is equivalent up to rigid transformation to $\{S\theta_i\}$ where $S$ is as in \cref{THM:WassmapDilationGeneral}, so the conclusion of \cref{THM:WassmapIsotropicDilation} is not contradictory.

We end this subsection by giving some concrete examples. The first is for the simple case when the density function of $\mu_0$ is symmetric, giving a concrete example of \cref{COR:WassmapDilation}.

\begin{proposition}\label{PROP:DilateSymmetry}
	Let $d\mu_0=f(x)dx$ be symmetric about the $x_1=x_2$ line, and let $d\mu_{\theta} =\det(D_\theta)f(D_\theta x) d$, where $D_\theta$ is as above. Then  \[W_2(\mu_\theta,\mu_{\theta'})^2 = [(\vartheta_1-\vartheta'_1)^2+(\vartheta_2-\vartheta_2')^2]\int_{x_2\geq x_1} (x_1^2+x_2^2)f(x)dx.\]
\end{proposition}

The proof of this proposition follows from direct calculation of the moments in \cref{COR:WassmapDilation} and so is omitted. The converse of \cref{PROP:DilateSymmetry} is not necessarily true. That is, the condition $W_2(\mu_\theta,\mu_{\theta'})^2 = c[(\vartheta_1-\vartheta_1')^2+(\vartheta_2-\vartheta_2')^2]$ for some $c\in \R$ does not imply that $\mu_0$ is symmetric across $x_1=x_2$. Indeed, consider the following example: suppose $d\mu_0 = \frac{1}{|A|}\mathbbm{1}_Adx$, where $A$ is a rectangle with range $(1,2)$ on $x_1$ axis and $(-1,3)$ on $x_2$ axis. Then \[W_2(\mu_\theta,\mu_{\theta'})^2 = \frac{7}{3}[(\vartheta_1-\vartheta_1')^2+(\vartheta_2-\vartheta_2')^2].\] 

For further illustration, the following corollary, easily obtained by computing the relevant second moments from \cref{THM:WassmapDilationGeneral}, shows what one recovers for a dilation manifold when the generating measure is the indicator function of a domain suitably normalized.

\begin{corollary}\label{COR:Rectangle}
Let $A$ be a rectangle in $\R^m$ with endpoints $a_{i,1},a_{i,2}$ on the $i$--th coordinate axis, and let $d\mu_0=\frac{1}{|A|}\mathbbm{1}_{A}dx$.  Then if $\theta,\theta'\subset\R^m_+$ are dilation vectors, 
\[W_2(\mu_\theta,\mu_{\theta'})^2 = \frac13 \sum_{i=1}^m|\vartheta_i-\vartheta_i'|^2(a_{2,i}^2+a_{2,i}a_{1,i}+a_{1,i}^2).\]
 Consequently, \cref{ALG:Wassmap} recovers $\{S_A\theta_i\}_{i=1}^N$, where $S_A$ is the diagonal matrix whose diagonal entries are defined as 
\[(S_A)_{i,i}= \sqrt{\frac{1}{3}(a_{2,i}^2+a_{2,i}a_{1,i}+a_{1,i}^2)}.
\]
\end{corollary}

Note that if the parameter set is a lattice in $\R^m$, i.e., $\Theta = \alpha_1\Z\times\cdots\times\alpha_m\Z$, then Functional Wassmap will recover the set $\alpha_1M_2^\frac12(P_1\mu_0)\Z\times\cdots\times\alpha_mM_2^\frac12(P_m\mu_0)\Z$ up to rigid transformation.

\subsection{Rotation manifolds}\label{SEC:Rotation}

We will show in subsequent experiments that the discrete Wassmap algorithm is capable of recovering the underlying circle governing a rotational manifold. However, at present, the authors do not have a proof analogous to the above results for this case.  Let a rotation manifold be defined as follows:

\[\mc{M}^{\textnormal{rot}}(\mu_0,\Theta) := \{\mu_0(R_\theta\cdot):R_\theta\in \textnormal{SO}(m)\}.\]

Consider the following: 

\begin{theorem}[{\cite[Theorem 1.22]{santambrogio2015optimal}}]\label{thm:KPexist}
	Suppose $\mu,\nu\in \mathbb{W}_2(\R^m)$ and $\mu$ gives no mass to $(d-1)$--surfaces of class $C^2$. Then there exists a unique optimal transport map $T$ from $\mu$ to $\nu$, and it is of the form $T=\nabla u$ for a convex function $u$.
\end{theorem}

A direct consequence of \cref{thm:KPexist} is that a rotation matrix $R_\theta$ is not the optimal transport map from a measure to its pushforward under $R_\theta$ as this is not the gradient of a convex function.  Consequently, exactly computing $W_2(\mu_0(R_\theta\cdot),\mu_0(R_{\theta'}\cdot))$ is nontrivial.

On the other hand, we can give an upper bound for the Wasserstein distance of a rotated version of a fixed measure with itself. Restrict to clockwise rotation in $\R^2$ by angle $\theta\in(0,2\pi)$, and let $R_\theta$ be the resulting rotation matrix. One can verify that
\begin{equation*}
	W_2(\mu_0,\mu_0(R_\theta\cdot))^2\leq \int_{\R^2}|R_\theta(x)-x|^2d\mu_0 = 2\sin\left(\frac{\theta}{2}\right)M_2(\mu_0).
\end{equation*}

\section{Discrete Wassmap: Algorithm and Theory}\label{SEC:DiscreteWassmap}

In imaging practice, one obtains discrete vectors rather than continuous distributions, so a practical version of \cref{ALG:Wassmap} must take this into account. To do this, one must consider how to form a probability measure from a given image. Given a two-dimensional (planar) or multidimensional (e.g., volumetric) image in pixel/voxel representation, that is $g = [g_1,\ldots,g_D]\in\R^D$ where $D$ is the total number of pixels or voxels, we will assign a discrete measure $P(g)\in \bb{W}_2(\R^m)$ by selecting a set of $D$ locations $x_n\in\R^m$, and assigning mass $g_n>0$ to a corresponding physical location $x_n$ and normalizing: 
\begin{equation}
    P(g) = \frac{1}{\|g\|_1}\sum_{n=1}^D g_n\delta_{x_n}, \label{eqn:image_to_discrete_measure}
\end{equation}where $\delta_{x_n}$ is a Dirac mass at location $x_n$. The locations $x_n$ are most conveniently assumed, at least initially, to lie on a regular grid in the ambient space $\R^m$.  
Given two images (with common ambient dimension $m$ but not necessarily the same $D$), the problem of computing the Wasserstein distance between $\mu_i = P(g_i)$ and $\mu_j = P(g_j)$ reduces to a discrete optimization problem for which many algorithms exist \cite{gerber2017multiscale,peyre2019computational}. 

Below we summarize the Discrete Wassmap algorithm which mimics the procedure described above. Note that we state the algorithm for image input, but one could equally well state it simply for discrete probability measures input in which case one simply skips the measure construction step.

\begin{algorithm}[h!]
 \caption{Discrete Wasserstein Isometric Mapping (Discrete Wassmap)}\label{ALG:DiscreteWassmap}
\begin{algorithmic}[1]

\STATE \textbf{Input: }{Image data $\{g_i\}_{i=1}^N\subset \R^D$; embedding dimension $d$
}

\STATE \textbf{Output: }{Low-dimensional embedding points $\{z_i\}_{i=1}^N\subset\R^d$}
\STATE {(Measure Construction): $\mu_i = P(g_i)$}
 \STATE {Compute pairwise Wasserstein distance matrix $W_{ij} = \edit{W_p}^2(\mu_i,\mu_j)$}
\edit{\STATE{(Optional) Form neighborhood graph $G$ using $W$, and set $W=\textnormal{APSP}(G)$}}
\STATE{$B = -\frac12 HWH$, where  ($H=I-\frac{1}{N}\mathbbm{1}_N)$}

\STATE {(Truncated eigendecomposition): $B_d=V_d\Lambda_dV_d^T$}

\STATE{$z_i = (V_d\Lambda_d^\frac{1}{2})(i,:) $} 

\STATE \textbf{Return:}{$\{z_i\}$}
\end{algorithmic}
\end{algorithm}

\subsection{Transferring Wasserstein computations to arbitrary measures}

An important consideration for the theory of exactness of Discrete Wassmap is to understand how (or even if) any of the Wasserstein distance computations in \cref{SEC:Theory} carry over to the setting of discrete measures. For instance, if one translates a discrete measure, is the Wasserstein distance the same as in the absolutely continuous case (the magnitude of the translation)?  Here we show that this is the case for a wide variety of discrete measures and transformations of them. We will state our results in terms of the pushforward operators defining the transformation of a base measure.

The following theorem provides a bridge which allows one to transfer results on recovery of Wasserstein image manifold parametrizations from manifolds generated by absolutely continuous measures to those generated by arbitrary measures. Note that there is no requirement on the generating measure $\mu_0$ aside from the fact that it lies in $\mathbb{W}_p$; it may have a mix of continuous and discrete spectra, and need not have compact support.  \Cref{THM:DiscreteContinuous} shows that if Wasserstein distances between absolutely continuous measures generated by a given parameter set depend only on the parameters, then the Wasserstein distances between arbitrary measures likewise depend only on the parameters. Thus, if Functional Wassmap recovers a parameter set for absolutely continuous generating measure $\mu_0$, then Discrete Wassmap recovers the manifold generated analogously from a discrete measure $\mu_0$. Additionally, since \cref{THM:DiscreteContinuous} holds for arbitrary measures, we find that if Functional Wassmap recovers a parameter set for absolutely continuous generating measure, then in fact it also recovers the parameter set for arbitrary generating measure. \edit{Below, we let $g_\sigma(x)=\frac{1}{(\sqrt{2\pi}\sigma)^m} e^{\frac{-|x|^2}{2\sigma^2}}$ be the multivariate Gaussian kernel on $\R^m$, and $\ast$ represents convolution of measures.}

\begin{theorem}\label{THM:DiscreteContinuous}
Let $p\in[1,\infty)$. Suppose that for all absolutely continuous $\mu_0\in\edit{\mathbb{W}_p}(\R^m)$, $\mathcal{M}(\mu_0,\Theta)=\{T_{\theta\sharp}\mu_0:\theta\in\Theta\}$ is a smooth submanifold of $\edit{\mathbb{W}_p}(\R^m)$, that $T_\theta$ is Lipschitz for all $\theta$, and that for all $\theta,\theta'\in\Theta$ and all absolutely continuous $\mu_0$, $W_p(T_{\theta\sharp}\mu_0,T_{\theta'\sharp}\mu_{0}) = f(\theta,\theta'\edit{,\mu_0})$ for some function $f$ dependent only upon $\theta, \theta',$ \edit{and $\mu_0$, for which $\lim_{\sigma\to0}f(\theta,\theta',\mu_0\ast g_\sigma)=f(\theta,\theta',\mu_0)$}.  Then for any $\nu_0\in\edit{\mathbb{W}_p}(\R^m)$, $\edit{W_p}(T_{\theta\sharp}\nu_0,T_{\theta'\sharp}\nu_{0})= f(\theta,\theta'\edit{,\nu_0})$.    
\end{theorem}

The crux of the proof of this theorem is the lemma below. We take $\nu_\sigma\rightharpoonup\nu$ to mean weak convergence of measures, which by the Portmanteau Theorem is equivalent to the statement $\nu_\sigma(A)\to\nu(A)$ for all continuity sets $A$ of $\nu$ (i.e., $\nu(\partial A)=0$).

\begin{lemma}\label{LEM:Discrete}
Let \edit{$\mu\in\mathbb{W}_p(\R^m)$, $p\in [1,\infty)$}. Suppose that $T_\theta,T_{\theta'}:\R^m\to\R^m$ are Lipschitz with Lipschitz constant at most $L$. Then \[\edit{W_p}(T_{\theta\sharp}\mu,T_{\theta'\sharp}\mu)=\lim\limits_{\sigma\rightarrow 0}\edit{W_p}(T_{\theta\sharp}(\mu*g_\sigma), T_{\theta'\sharp}(\mu*g_\sigma)).\] 
\end{lemma}

\begin{proof}
By \cite[Lemma 5.2]{santambrogio2015optimal}, $\edit{W_p}(T_{\theta\sharp}\mu,T_{\theta'\sharp}\mu)=\lim\limits_{\sigma\rightarrow 0}\edit{W_p}(T_{\theta\sharp}\mu*g_\sigma, T_{\theta'\sharp}\mu*g_\sigma)$. While it is not true in general that $\edit{W_p}(T_{\theta\sharp}\mu\ast g_\sigma,T_{\theta'\sharp}\mu\ast g_\sigma) = \edit{W_p}(T_{\theta\sharp}(\mu*g_\sigma), T_{\theta'\sharp}(\mu*g_\sigma))$, we will show that the limits of these two expressions is the same. Considering their difference and utilizing the \edit{reverse triangle inequality then triangle inequality} below, we have
\begin{multline*}
|\edit{W_p}(T_{\theta\sharp}\mu*g_\sigma, T_{\theta'\sharp}\mu*g_\sigma)- \edit{W_p}(T_{\theta\sharp}(\mu*g_\sigma),T_{\theta'\sharp}(\mu*g_\sigma))| \\
\edit{\leq |W_p(T_{\theta\sharp}\mu*g_\sigma, T_{\theta'\sharp}\mu*g_\sigma)- W_p(T_{\theta'\sharp}\mu*g_\sigma, T_{\theta\sharp}(\mu*g_\sigma))|} \\ \edit{ \qquad + |W_p(T_{\theta'\sharp}\mu*g_\sigma, T_{\theta\sharp}(\mu*g_\sigma))-W_p(T_{\theta\sharp}(\mu*g_\sigma),T_{\theta'\sharp}(\mu*g_\sigma))|}\\
\leq \edit{W_p}(T_{\theta\sharp}\mu*g_\sigma,T_{\theta\sharp}(\mu*g_\sigma)) + \edit{W_p}(T_{\theta'\sharp}(\mu*g_\sigma), T_{\theta'\sharp}\mu*g_\sigma) \\
\leq \edit{W_p}(T_{\theta\sharp}\mu,T_{\theta\sharp}(\mu*g_\sigma))+\edit{W_p}(T_{\theta\sharp}\mu,T_{\theta\sharp}\mu*g_\sigma) \\ + \edit{W_p}(T_{\theta'\sharp}\mu,T_{\theta'\sharp}(\mu*g_\sigma))+\edit{W_p}(T_{\theta'\sharp}\mu,T_{\theta'\sharp}\mu*g_\sigma).
\end{multline*}

We claim that $\lim\limits_{\sigma\rightarrow0}\edit{W_p}(T_{\theta\sharp}\mu,T_{\theta\sharp}(\mu*g_\sigma))=\lim\limits_{\sigma\rightarrow0}\edit{W_p}(T_{\theta\sharp}\mu,T_{\theta\sharp}\mu*g_\sigma)=0$ for any $\theta$. By \cite[Lemma 5.11]{santambrogio2015optimal}, it is sufficient to prove the following:  
\begin{enumerate}
\item $T_{\theta\sharp}(\mu*g_\sigma)\rightharpoonup T_{\theta\sharp}\mu$
\item $T_{\theta\sharp}\mu*g_\sigma \rightharpoonup T_{\theta\sharp}\mu$.
\item \edit{$\int|x|^pdT_{\theta\sharp}(\mu*g_\sigma)\rightarrow \int |x|^pdT_{\theta\sharp}\mu$} 
\item \edit{$\int|x|^pdT_{\theta\sharp}\mu*g_\sigma\rightarrow \int |x|^pdT_{\theta\sharp}\mu$.}
\end{enumerate}
Here and subsequently all integrals are over $\R^m$. \edit{For 3) and 4), we will first prove the statements for integer $p$ and show that it can be generalized to any $p\in [1,\infty)$.}

Proof of 1) Suppose $A$ is a \edit{measurable} set of $\mu$, then by definition $T^{-1}(A)$ is a continuity set of $T_{\theta\sharp}\mu$.	Then,
\[T_{\theta\sharp}(\mu*g_\sigma)(A) = \mu*g_\sigma(T^{-1}(A)) \to \mu(T^{-1}(A)) = T_{\theta\sharp}\mu(A),\]
where convergence follows from the fact that $\mu\ast g_\sigma\rightharpoonup \mu$ for any $\mu$.
	
Item 2 is well-known and follows from a simple computation, so is omitted. 

Proof of 3) First, we note that by direct computation,
\[\int|x|^p dT_{\theta\sharp}(\mu\ast g_\sigma) = \int\int|T(x+y)|^pg_\sigma(y)dy d\mu(x)\] and
\[\int|x|^pdT_{\theta\sharp}\mu = \int|T(x)|^pd\mu(x) = \int\int|T(x)|^pg_\sigma(y)dy d\mu(x),\]
where the second equality follows from $g_\sigma$ being a probability density function.

With these observations in hand, \edit{if p is an integer,} we have that the difference \edit{$|\int |x|^p d(T_{\theta\sharp}(\mu*g_\sigma)-T_{\theta\sharp}\mu)|$ is bounded as follows:
\begin{align*}
		& 		\int\int\big||T(x+y)|^p-|T(x)|^p\big|g_\sigma(y)dy d\mu(x)\\
		 & = \int\int(|T(x+y)-T(x)|)\sum\limits_{i=0}^{p-1} |T(x+y)|^i|T(x)|^{p-1-i} g_\sigma(y)dy d\mu(x)\\
         &  \leq \int\int |y|\sum\limits_{i=0}^{p-1} |T(x+y)-T(0)+T(0)|^i|T(x)-T(0)+T(0)|^{p-1-i} g_\sigma(y)dy d\mu(x)\\
		  & \leq L^{p-1}\int\int 
		  \bigg(\sum\limits_{i=0}^{p-1} \big(|x+y|+T(0)\big)^i\big(|x|+T(0)\big)^{p-1-i}\bigg)|y|g_\sigma(y)dy d\mu(x).
	\end{align*}
}




The \edit{second inequality follows from utilizing the fact that $T$ is Lipschitz. The final integral can be bounded by a sum of integrals involving products of $|x|^{p-1}$ and $|y|^\alpha$ for various integer $\alpha\in[1,p]$. } That this quantity goes to $0$ as $\sigma\to0$ follows from the facts \edit{that $\mu$ has finite $p-1$ moment} and that any Gaussian moment tends to $0$. Indeed, by substitution,
\[\int_{\R^m}|y|^pg_\sigma(y)dy = \sigma^p\int_{\R^m}|y|^pg_1(y)dy.\]
The integral above is a constant depending only upon $p$ and $m$, so the conclusion follows by application of the Dominated Convergence Theorem.

Proof of 4) By similar argument, but noting that
\[\int|x|^pd(T_{\theta\sharp}\mu\ast g_\sigma) = \int\int|T(x)+y|^pg_\sigma(y)dy d\mu(x),\]
we see that \edit{for integer $p$,
\begin{align*}
		\bigg|\int |x|^p d(T_{\theta\sharp}\mu*g_\sigma-T_{\theta\sharp}\mu)\bigg|
		& \leq \int\int\big||T(x)+y|^p-|T(x)|^p\big|g_\sigma(y)dy d\mu(x)\\
		& \leq \int\int\Big(\big||T(x)|+|y|\big|^p-|T(x)|^p\Big)g_\sigma(y)dy d\mu(x)\\
		& = \int\int \Big(\sum\limits_{i=0}^{p-1}{p\choose i}|T(x)|^i|y|^{p-i}\Big) g_\sigma(y)dyd\mu(x)\\
		& \leq \int\int \Big(\sum\limits_{i=0}^{p-1}{p\choose i}L^i\big(|x|+T(0)\big)^i|y|^{p-i}\Big) g_\sigma(y)dyd\mu(x).
	\end{align*}
}
The \edit{finite} moment of $g_\sigma$ tends to $0$ as before.

\edit{For non-integer $p$, the conclusion follows from the fact that $|a^p-b^p|\leq |a^{\lfloor p \rfloor}-b^{\lfloor p \rfloor}|+|a^{\lceil p \rceil}-b^{\lceil p \rceil}|$ and $\mu\in\mathbb{W}_p$ has finite $\lfloor p\rfloor$-th moment.}
\end{proof}

With this Lemma we are now in a position to finish the proof of the main theorem of this section.

\begin{proof}[Proof of \cref{THM:DiscreteContinuous}]
Let $g_\sigma$ be the multivariate Gaussian with variance $\sigma$ as before. Then by \cref{LEM:Discrete}, we have
\begin{align*} \edit{W_p}(T_{\theta\sharp}\nu_0,T_{\theta'\sharp}\nu_0) = \lim_{\sigma\to0} \edit{W_p}(T_{\theta\sharp}(\nu_0\ast g_\sigma),T_{\theta'\sharp}(\nu_0\ast g_\sigma)) & = \edit{\lim_{\sigma\to0} f(\theta,\theta',\nu_0\ast g_\sigma)} \\ & \edit{= f(\theta,\theta',\nu_0).} \end{align*}
\end{proof}

By combining the above results, we readily see that Discrete Wassmap recovers the translation set for translation of discrete measures, and the scaled dilation set for dilation image manifolds.

\begin{corollary}\label{COR:TransDil}
Suppose $\mu_0\in\mathbb{W}_2(\R^m)$ is discrete. Then given $\{\theta_i\}_{i=1}^N\subset\Theta\subset\R^m$ and corresponding measures $\{\mu_{\theta_i}\}_{i=1}^N\subset\mathcal{M}^{\text{trans}}(\mu_0,\Theta)$, the Discrete Wassmap algorithm (\cref{ALG:DiscreteWassmap}) with embedding dimension $m$ recovers $\{\theta_i\}_{i=1}^N$ up to rigid transformation.

If $\mu_0\in\mathbb{W}_2(\R^m)$ is discrete. Then given $\{\theta_i\}_{i=1}^N\subset\Theta\subset\R^m_+$ and corresponding measures $\{\mu_{\theta_i}\}_{i=1}^N\subset\mathcal{M}^{\textnormal{dil}}(\mu_0,\Theta)$, then the Discrete Wassmap algorithm with embedding dimension $m$ recovers $\{S\theta_i\}_{i=1}^N$ up to rigid transformation, where $S$ is as in \cref{THM:WassmapDilationGeneral}.
\end{corollary}

\begin{proof}
Combine \cref{THM:DiscreteContinuous} with \cref{THM:WassmapTranslation,THM:WassmapDilationGeneral}.
\end{proof}

\begin{remark}
Similarly, if $\mu_0 \in\mathbb{W}_2(\R^m)$ is arbitrary, then \cref{THM:DiscreteContinuous,THM:WassmapTranslation,THM:WassmapDilationGeneral} imply that the Functional Wassmap algorithm (\cref{ALG:Wassmap}) recovers the underlying translation or scaled dilation sets, respectively. \edit{Additionally, the conclusion of \cref{COR:TransDil} holds for general $p\in(1,\infty)$ in the case of translations or isotropic dilations.}
\end{remark}

\section{Experiments}\label{SEC:Experiments}

To demonstrate our theoretical results, we provide several experiments\footnote{Code for this work may be found at \url{https://github.com/Wassmap/wassmap}.} using both synthetically generated two-dimensional image data and the standard MNIST digits dataset \cite{lecun1998mnist}.  For each synthetic experiment, a fixed absolutely continuous base measure $\mu_0 \in \bb{W}_2(\R^2)$ with density $f_0(x)$ is selected, then a manifold $\mc{M}(\mu_0,\Theta)$ is sampled by applying the parametric transformation $\mc{T}_\theta$ to $\mu_0$ for a finite number of $\theta$ values $\{\theta_1,\ldots,\theta_N\}\subset\Theta$, resulting in the measures $\mu_{\theta_i}$ and corresponding densities $f_{\theta_i}$.  These (continuum) images are subsequently discretized by performing a spatial sampling, selecting $\{x_1,\ldots,x_D\}\subset\R^2$, evaluating each density $f_{\theta_i}(x)$ at these points, then forming the discrete measure \eqref{eqn:image_to_discrete_measure}.  

\edit{Comparisons are shown to Euclidean MDS, Isomap, and Diffusion Maps. MDS and Isomap are the nearest, most faithful comparison to our method as they are also global algorithms, but the other methods mentioned are local methods, which we employ for a more comprehensive comparison.  For Isomap, the $k$NN graph is used for geodesic estimation, with $k$ tuned to give the best visual result. For Diffusion Map embeddings, we employ the standard Gaussian kernel with the automatic epsilon selection algorithm described in \cite{berry2015nonparametric}. Our experiments use the Wassmap variant without the graph geodesic computation step with the exception of the MNIST experiment in \cref{SEC:MNIST} (Figure \ref{FIG:MNIST01}). For further illustration of Wassmap with and without use of geodesics, see the supplemental material. We use Euclidean distances in all methods except for Wassmap.}  Note that methods other than Wassmap assume a `pixel' representation of images; that is, each image is treated as an element of $\R^D$ for some fixed $D$.  One can obtain such a representation by following the steps outlined above but keeping points of zero density (for more discussion, see the supplemental material and code). In all figures, points in the original parameter set are color coded, and the corresponding points in the embedding depend on the initial point color in a one-to-one correspondence. Thus one can see where initial parameter points end up in each embedding. 

\subsection{Translation manifold}
In this set of experiments, we take the base measure $\mu_0$ to be the indicator function of a disc of radius $1$, that is $d\mu_0 = \frac{1}{\pi}\mathbbm{1}_{D}(x)dx$. For a given translation set $\Theta\subset\R^2$, the translation manifold is then generated via \eqref{eqn:trans_manifold_def}.  We consider two translation sets: $\Theta_1 = [-1,1]^2$ and $\Theta_2 = [-10,-5]\times[-2.5,2.5]\cup [5,10]\times[-2.5,2.5]$.   

\begin{figure}[h!]
    \centering
    \includegraphics[width=0.9\textwidth]{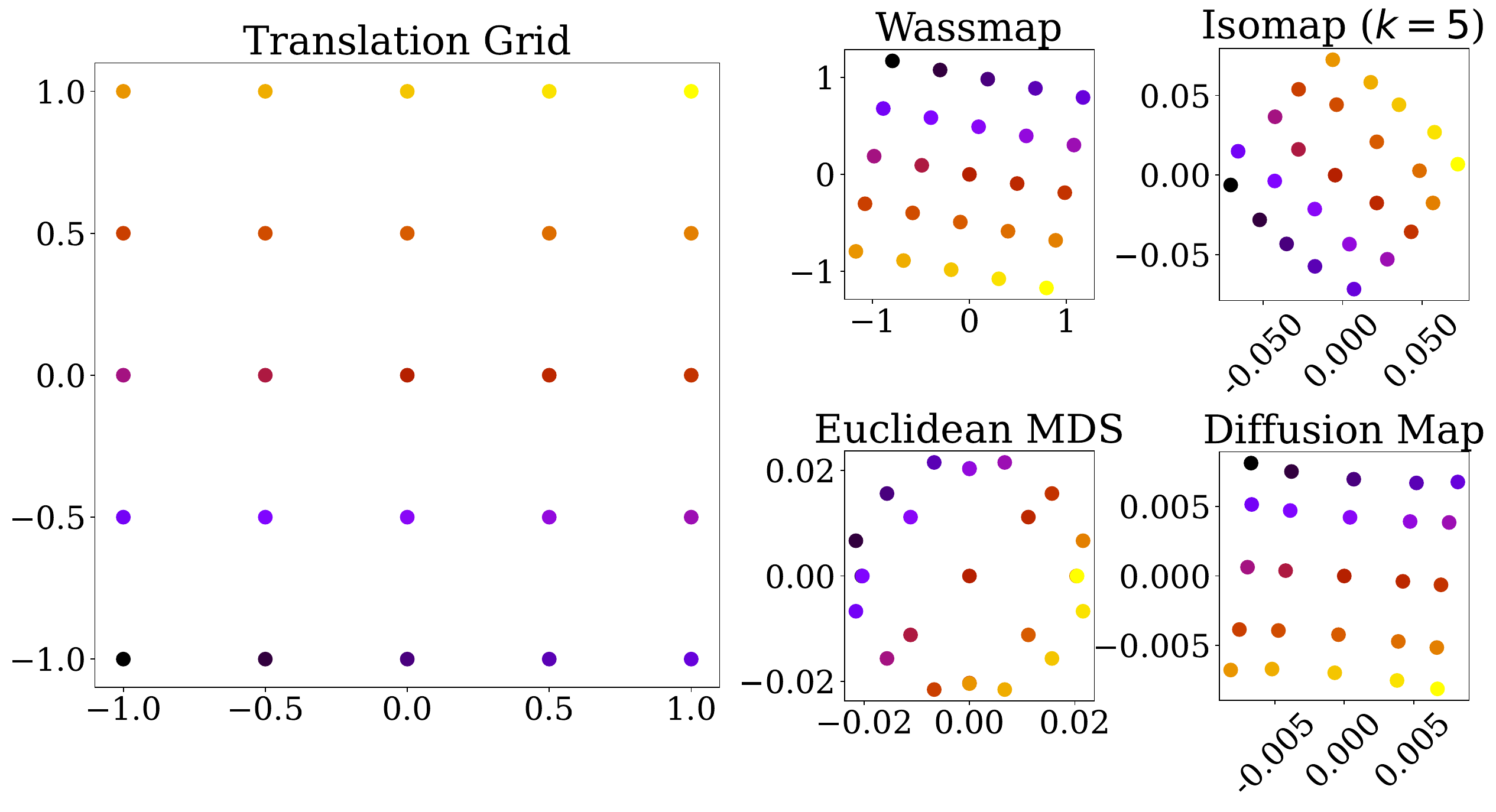}
    \caption{\edit{Translation manifold generated by the characteristic function of the unit disk with parameter set $\Theta_1 = [-1,1]^2$. We consider a uniform $5\times 5$ grid in the parameter space to generate $\{\theta_i\}$.  Shown are the original translation grid, the Wassmap, Isomap, Euclidean MDS,  and Diffusion Map embeddings.  For Isomap, the number of neighbors $(k=5)$ that resulted in the best embedding was chosen.}}
    \label{FIG:Translation}
\end{figure}


\begin{figure}[h!]
    \centering
    \includegraphics[width=0.9\textwidth]{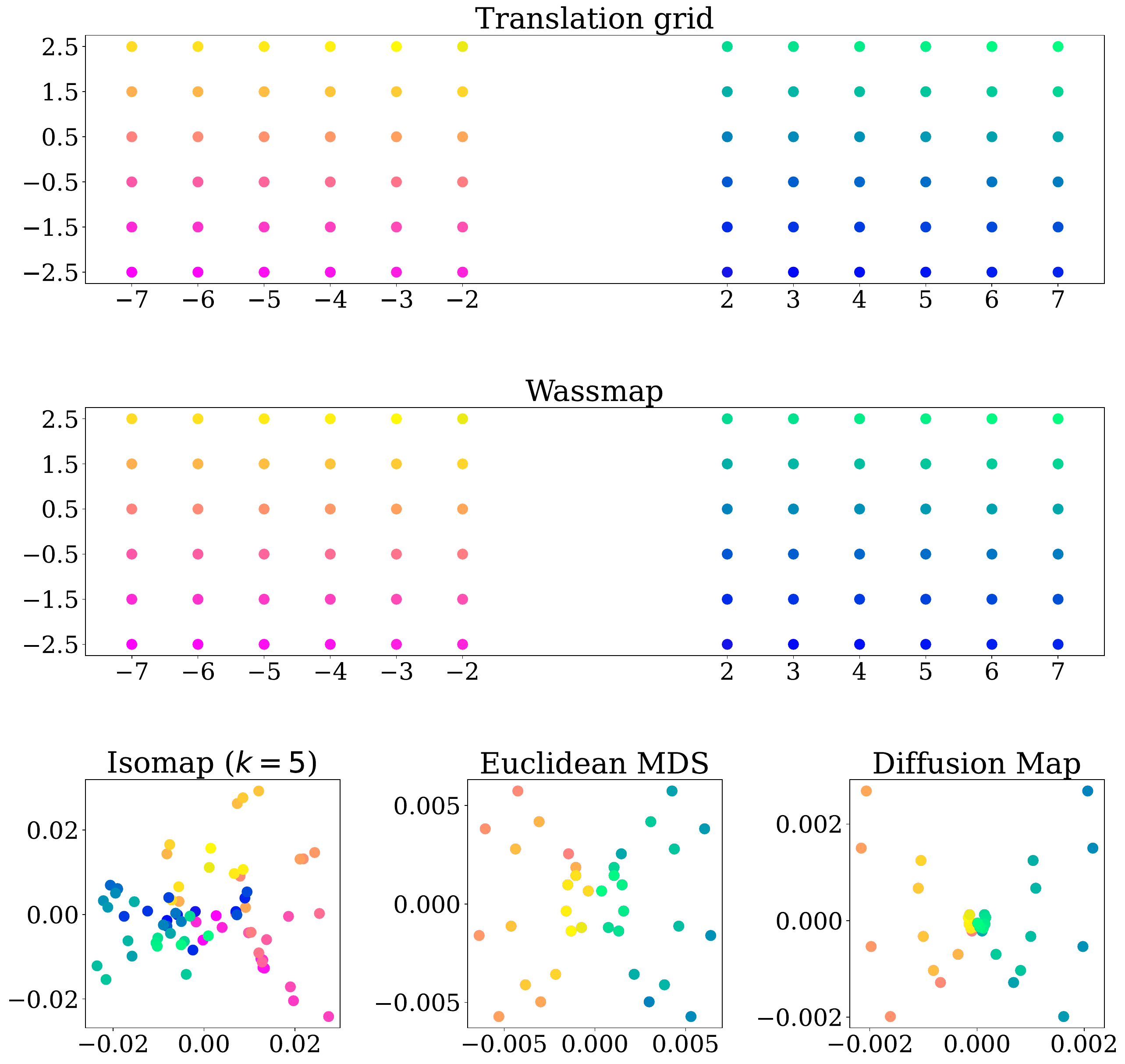}
    \caption{\edit{Translation manifold generated by the characteristic function of the unit disk with parameter set $\Theta_2$. We consider a $6\times 6$ grid in each disjoint piece of the parameter space to generate $\{\theta_i\}$.  Shown are the original translation grid, Wassmap embedding, Isomap, Euclidean MDS, and Diffusion map embeddings. Note that both Euclidean MDS and Diffusion Map have duplicated embedding points (explaining the apparent lack of some points in these embeddings).}}
    \label{FIG:TranslationNonconvex}
\end{figure}

Both \cref{FIG:Translation,FIG:TranslationNonconvex} show that Wassmap recovers the underlying translation grid up to rigid motion as predicted by \cref{THM:WassmapTranslation}; in \cref{FIG:Translation} a rotation appears, but the side-lengths of the embedded grid are 2 as in the original parameter set $\Theta_1$.  In both experiments, the translated discs overlap; consequently, the other methods produces embeddings that appear coherent despite failing to recover the parameter set exactly. In particular, the other embeddings of $\Theta_1$ appear to be morphed grids, and the scale is dilated in a way that the Wassmap embedding is not. Two advantages of Wassmap in this case are that it is not subject to careful parameter tuning, and the size of the disk and its relation to the parameter grid does not matter, whereas for other methods parameter tuning may play an important role, and the translates of the disk must overlap significantly for the pairwise Euclidean distances to be meaningful.

\Cref{FIG:TranslationNonconvex} illustrates that Wassmap is capable of recovering nonconvex translation parameter sets in contrast to both discrete and continuum Isomap \cite{donoho2005image}. \edit{The Isomap embedding is largely incoherent in this case, while the Euclidean MDS and Diffusion Map embeddings exhibit significant skewing as well as overlapping points in the embedding.}

\subsection{Dilation manifold}
To illustrate the case of dilations, we consider the same base measure as in the translation case (disc support function centered at $(0,0)$), but now apply the anisotropic dilation transformation $D_\theta$ as discussed in \cref{subsec:dilationthy}, where the parameters $\vartheta_1,\vartheta_2$ come from a regular 4x4 subgrid of $[0.5,2]\times[0.5,4]$.  The dilation parameter grid and result of the different embeddings are shown in \cref{FIG:Dilation}. \edit{The Euclidean MDS embedding is the next best compared to Wassmap, the latter of which recovers the structure of the parameter set faithfully, but all other embeddings are relatively poor.} Note that the dilation grid has size $1.5\times3.5$, and the Wassmap embedding has size approximately $1.75\times .75$. One can compute the projected second moment of the base measure $d\mu_0 = \frac{1}{\pi}\mathbbm{1}_D dx$ as $(M_2(P_i\mu_0))^\frac12 = \frac{1}{2}$. \Cref{COR:WassmapDilation} states that Wassmap recovers the dilation grid up to this factor and a global rotation.  Thus we see that the Wassmap embedding does recover the original dilation grid multiplied by $0.5$ (the moment term) rotated by $\pi/2$.

\begin{figure}[h!]
    \centering
    \includegraphics[width=0.9\textwidth]{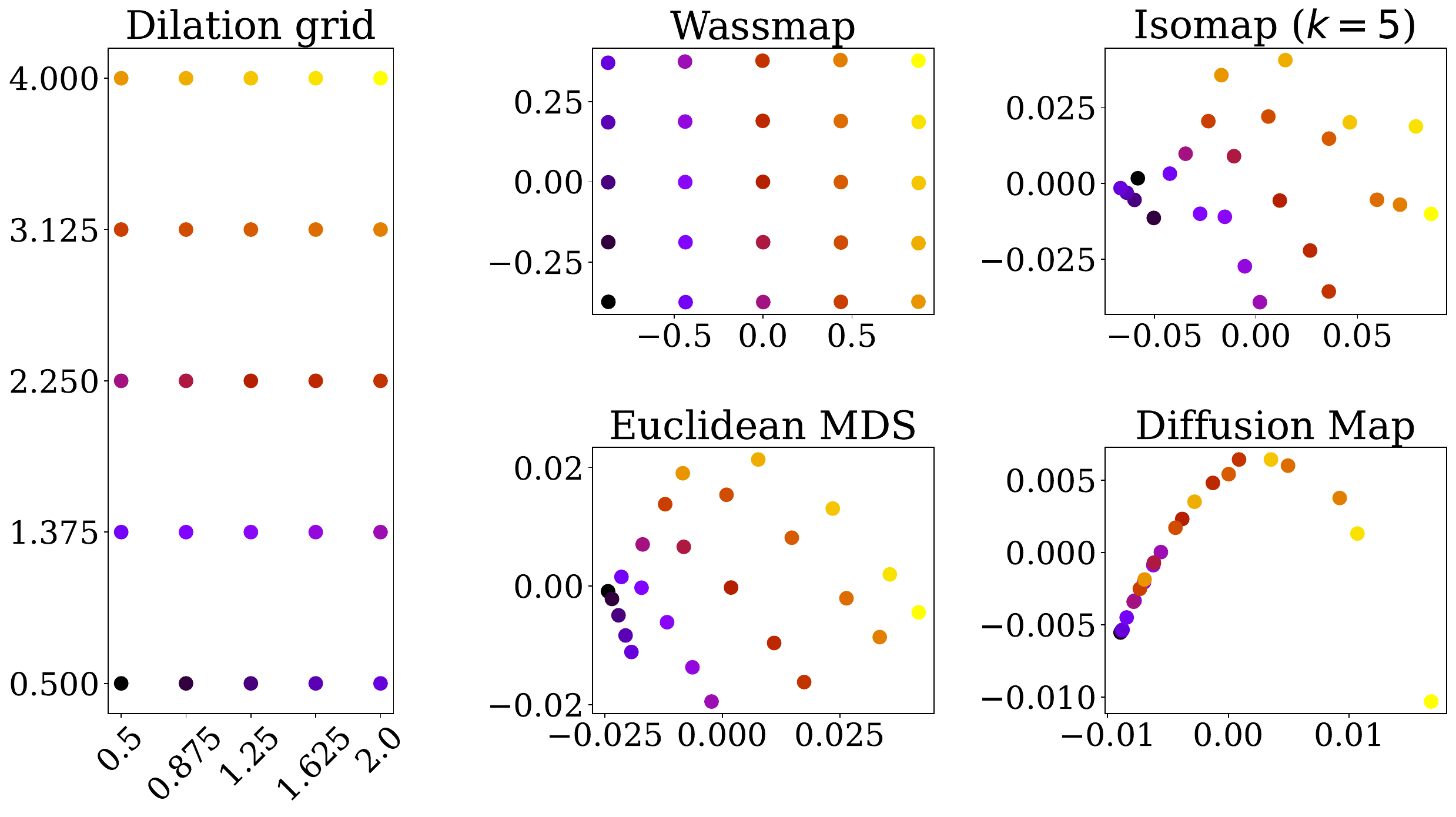}
    \caption{\edit{Dilation manifold generated by the characteristic function of the unit disk with parameter set $\Theta_3=[0.5,2]\times[0.5,4]$. We consider a uniform $5\times 5$ grid to generate $\{\theta_i\}$. Shown are the original dilation grid, the Wassmap, Isomap, Euclidean MDS, and Diffusion map embeddings.}}\label{FIG:Dilation}
\end{figure}

\subsection{Rotation manifold} To illustrate the case of rotational manifolds, we consider the base measure $\mu_0$ as the indicator of an ellipse with major radius 1 and minor radius 0.5, centered at $(x,y)=(0,1)$ This measure is rotated about the origin to obtain the sampled manifold at uniform angles $\theta_i\in[0,2\pi]$; embedding results are shown in \cref{FIG:RotationCentered}.  We see that \edit{all methods approximately or exactly recover a circular manifold.} This experiment provides evidence that Wassmap is capable of recovering rotational manifolds, though at present we are not able to prove this as discussed in \cref{SEC:Rotation}. 


\begin{figure}[h!]
    \centering
    \includegraphics[width=0.9\textwidth]{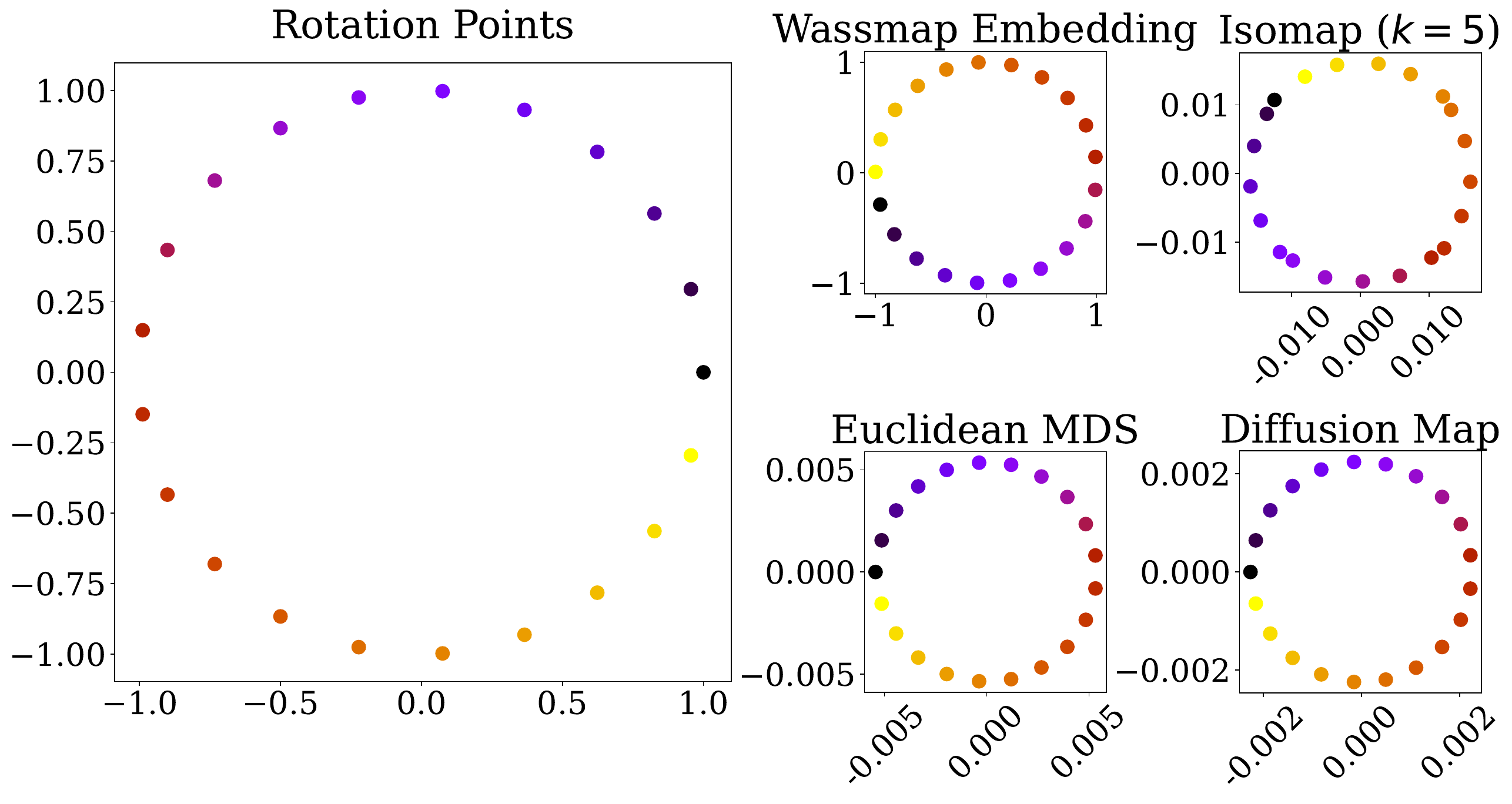}
    \caption{\edit{Rotation manifold generated by the characteristic function of an ellipse with major radius 1 and minor radius 0.5 centered at $(0,1)$. Rotation angles are uniformly sampled between $0$ and $2\pi$. Shown are the original points on the circle $(\cos\theta_i,\sin\theta_i)$, the Wassmap embedding, the Isomap, Euclidean MDS, and Diffusion Map embeddings.  Note that while the other methods produce circular embeddings, only Wassmap exactly reconstructs the unit circle.}}\label{FIG:RotationCentered}
\end{figure}



\subsection{Embedding MNIST}\label{SEC:MNIST}

\edit{Here we show the effect of Wassmap on embedding MNIST handwritten digits \cite{lecun1998mnist}.  We randomly sample 250 handwritten 1s, 2s, 7s and 9s from MNIST and compute the 2-dimensional Wassmap embeddings corresponding to two different choices of $k$ when forming the $k$NN graph.  We also computed the Isomap embeddings based on a $k$NN graph with the same $k$; recall that for both Wassmap and Isomap, the large $k$ limit corresponds to skipping the Geodesic computation step and using the raw pairwise distances.  \Cref{FIG:MNIST01} shows the resulting embeddings, with Euclidean MDS and Diffusion Maps also shown for comparison.   Note that Wassmap produces an embedding for which, for any value of $k$, the classes are relatively easily separated by a kernel SVM or nearest neighbor classifier, whereas the Isomap embedding appears sensitive to the choice of $k$ and results in nontrivial class overlaps for any $k$.  Both the Euclidean MDS and Diffusion Map embeddings also struggle to separate classes, particularly 9s and 7s.  A full analysis of classification performance and its dependence on the embedding dimension $d$ is a subject of future work.}

\begin{figure}[h!]
    \centering
    \includegraphics[width=0.9\textwidth]{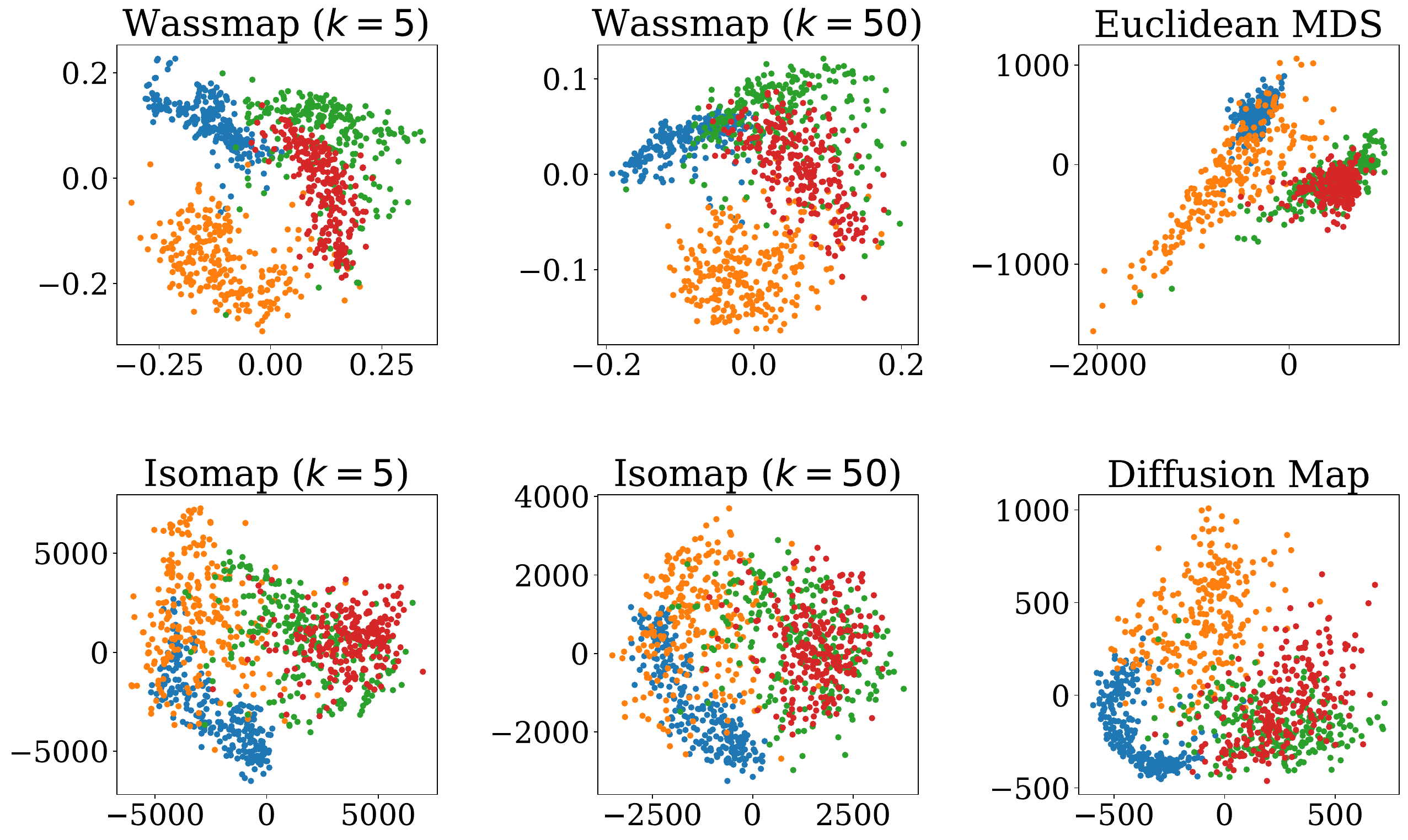}
    \caption{\edit{Random sample of 250 1s (blue), 2s (orange), 7s (green) and 9s (red) from MNIST. On the top row are shown the Wassmap embeddings for two values of the $k$NN parameter along with the Euclidean MDS embedding.  On the bottom row, Isomap embeddings with two $k$NN values are shown along with a Diffusion Map embedding.  }}
    \label{FIG:MNIST01}
\end{figure}

{\color{black}
\section{Computational aspects}\label{SEC:Computation}

From a practical perspective, the biggest drawback of the proposed approach is that of computational cost. Here, we discuss several relaxations and approximations that can be done that speed up Wassmap.

\subsection{Approximations of Wasserstein distances}

The experiments done here used the exact linear program solver to compute Wasserstein distances, but one can trade off accuracy of the embedding with speed of computation by utilizing approximation algorithms that approximate each of the Wasserstein distances. In general, an exact Wasserstein distance computation between discrete measures with $n$ points of mass carries complexity $\Omega(n^3\log n)$ without enforcing additional constraints on the measures \cite{pele2009fast}.  However, one can utilize any Wasserstein distance approximation in the Wassmap framework; some examples are entropic regularization and Sinkhorn distances \cite{altschuler2019massively,bonafini2021domain,cuturi2013sinkhorn,lin2022efficiency,schmitzer2019stabilized}, multiscale methods \cite{gerber2017multiscale,glimm2013iterative,schmitzer2016sparse}, and linearization techniques \cite{greengard2022linearization}. Sinkhorn distances can be computed in $O(n^2)$ time, but many approximate Wasserstein distance methods do not have concrete computational complexities, which makes it difficult to write down an overall complexity for Wassmap.

\begin{figure}
    \centering
    \includegraphics[width=.8\textwidth]{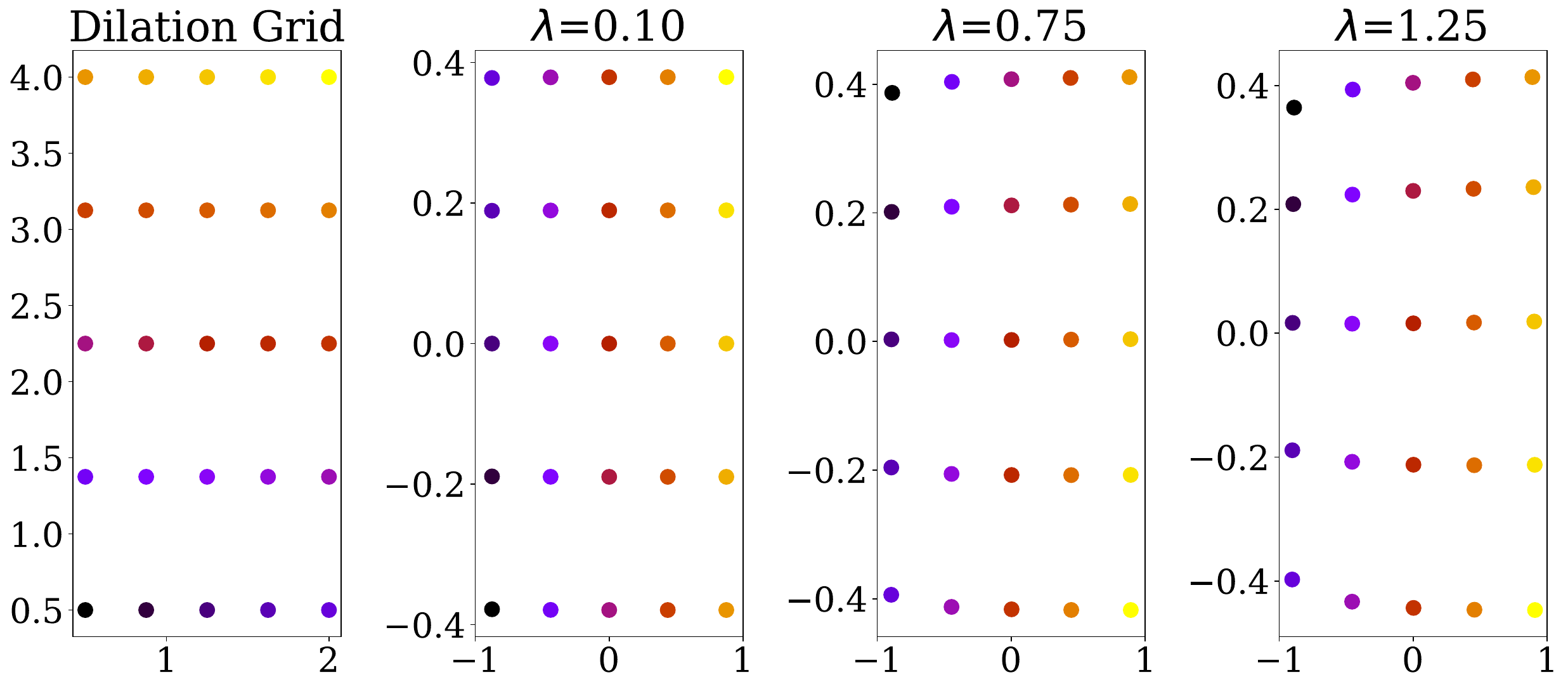}
    \caption{Dilation experiment with the same setup as in \cref{FIG:Dilation}. The Wasserstein distance matrix is approximated by the Sinkhorn distance with different regularization parameters $\lambda$.  The right embedding took about a third as long as the standard Wassmap embedding.}
    \label{FIG:Sinkhorn}
\end{figure}

\subsection{Approximating the MDS kernel matrix}

One note regarding \cref{ALG:Wassmap} is that the Wasserstein distance computations can be done in parallel, which can greatly speedup the computation time, although computing $O(N^2)$ Wasserstein distances is still prohibitive for large $N$.

One could use the Linear Optimal Transport (LOT) framework developed by Wang et al. \cite{wang2013linear} and furthered by several others \cite{aldroubi2021partitioning,khurana2022supervised,kolouri2017optimal,moosmuller2020linear,park2018cumulative} as an alternative. Cloninger and Moosm\"{u}ller show that for some types of object manifolds, the LOT distance is equal to the Wasserstein distance between manipulated images, hence our theoretical recovery results could be obtained for LOT distances in these cases as well.  Here, the LOT distance is defined as the $L_2$--norm between the transport (Monge) maps from each image to a fixed reference measure (e.g., a Gaussian). In this way, one could compute $O(N)$ LOT distances rather than $O(N^2)$ Wasserstein distances prior to MDS.  One potential drawback of this approach is that for more curved manifolds in Wasserstein space, or for diffeomorphisms that do not satisfy the compatibility condition of \cite{moosmuller2020linear}, the LOT approximation is not exact, and the embedding may suffer from this. In followup work to this paper \cite{cloninger2023linearized} we explored further the use of LOT in the Wassmap framework, and provide a computationally efficient LOT Wassmap algorithm as well as corresponding approximation guarantees which estimate how far the approximate embedding is from the Wassmap embedding when using LOT distances, empirical samples from data measures, and entropic regularization to compute transport maps.

An alternative would be to use the Nystr\"{o}m method \cite{gittens2013revisiting,williams2001using} to approximate the squared distance matrix ($B$) of MDS directly. A Nystr\"{o}m approximation of a symmetric matrix is of the form $B\approx CW_r^\dagger C^T$ where $C = W(:,I)$ is a column submatrix of $B$, $W=B(I,I)$ is the intersection of $C$ and $C^T$, and $W_r^\dagger$ is the Moore--Penrose pseudoinverse of $W_r$, which is the truncated SVD of $W$ of rank $r$. It is known that in cases where a kernel matrix is incoherent and approximately low-rank, a rank $r$ Nystr\"{o}m approximation of an $N\times N$ matrix can yield a good approximation by only computing $O(r\log N)$ columns of the kernel matrix. For the Wasserstein distance matrix, this would allow one to perform only $O(rN\log N)$ Wasserstein distance computations, which is slightly more than using LOT but considerably less than computing the full distance matrix. As an illustration, we rerun the dilation experiment above but with a $12\times12$ grid in the parameter space (\cref{FIG:Nystrom}), and we sample 4 out of 144 ($\log N$) columns to form the Nystr\"{o}m approximation to the Wasserstein distance matrix, which we then embed via MDS. For each Wasserstein computation we use the Sinkhorn algorithm of Cuturi. One can see that the shape of the embedding remains, though there are some errors.

\begin{figure}[h!]
    \centering
    \includegraphics[width=.9\textwidth]{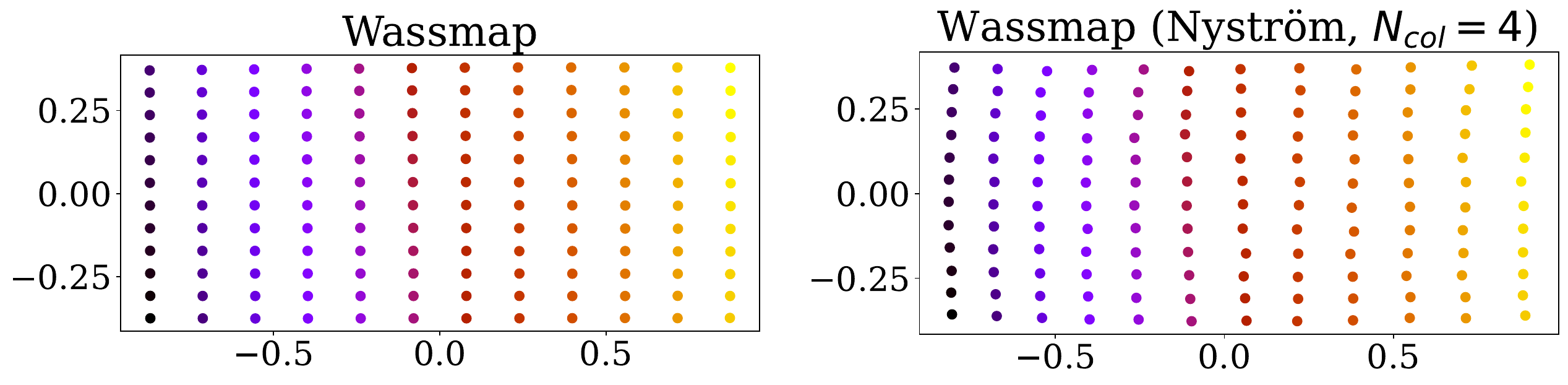}
    \caption{Dilation experiment with the same parameter space and base measure as in \cref{FIG:Dilation}, but with a $12\times12$ grid. The Wasserstein distance matrix is approximated by a Nystr\"{o}m approximation with only 4 out of 144 columns being computed.}
    \label{FIG:Nystrom}
\end{figure}

Note that this Nystr\"{o}m approximation is of the squared Wasserstein distance matrix in \cref{ALG:DiscreteWassmap}, which is different than the method of Altschuler et al.~\cite{altschuler2019massively}, which uses the Nystr\"{o}m method to approximate the cost matrix for a single Wasserstein distance computation between a pair of measures. These methods can be combined to yield a faster, more scalable approximation of our Wassmap algorithm.

\subsection{Updating an MDS embedding}

In practice, if one already has a low-dimensional embedding via MDS, one might want to update the embedding when a new image is obtained. This can easily be done in the following way: first compute the Wasserstein distance from the new image to the other images (this requires $N$ $W_2$ computations), and perform a rank-one update of the SVD of the matrix $B$ in \cref{ALG:Wassmap} (or \cref{ALG:DiscreteWassmap}), which can be done in $O(Nd+d^3)$ flops \cite{brand2006fast}, which is much faster than recomputing a new SVD of $B$ at a cost of $O(N^2d)$ (here $d$ is the embedding dimension, which is the rank of the embedding matrix in MDS).

\subsection{Choosing the embedding dimension}

In any method requiring choice of embedding dimension, there are several methods one can employ. For example, a scree plot of the singular values of the distance matrix can give an idea by looking for an elbow, or bend in the graph. More sophisticated techniques involve estimating the dimension of the manifold by local PCA or something similar \cite{little2009multiscale,little2009estimation,little2017multiscale}.
}

\section{Conclusion and Future Outlook}

This paper proposed the use of Wasserstein distances in the Isomap algorithm \edit{(and its precursor MDS)} as a more suitable measure of distance between images. The resulting Wassmap algorithm and its variants were shown to recover (up to rigid transformation) several parametrizations of image manifolds, including translation and dilation sets. We provided a bridge which transfers functional manifold recovery results to discrete recovery, which illustrate that the Discrete Wassmap algorithm recovers parametrizations of image manifolds generated by discrete measures. The practical experiments illustrate the effectiveness of the proposed framework on various synthetic and benchmark data. There is more to be explored regarding Wassmap, including its potential to recover rotation manifolds, those generated by composition of different operations (e.g., translation plus dilation or rotation), and manifolds generated by some class of parametrized diffeomorphisms acting on one or multiple generators. It also remains to explore the effects of additive noise ($\eta$ in \eqref{EQN:Imaging}) and the structure of the imaging operator $\mathcal{H}$.

Future work will also explore the use of Wasserstein distances in other manifold learning paradigms, including local methods such as LLE and tSNE, as well as use of $W_p$ for other $p\in[1,\infty)$ (for example, Kileel et al.~\cite{kileel2021manifold} approximate the classic Earthmover's Distance $W_1$ in the Laplacian eigenmap setting). One could easily replace Euclidean distances with Wasserstein distances in a na\"{i}ve way in any manifold learning algorithm such as Laplacian eigenmaps or Diffusion Maps. We have done this in a simple example in the Supplemental Material here, but the results are not competitive with other methods for this particular task of parametrization recovery. Better understanding how to use these algorithms for data coming from submanifolds of Wasserstein space would be an interesting avenue of future study.  

\edit{Additional studies will be done on choosing $\eps$ adaptively for the neighborhood graph step \cite{mekuz2006parameterless}, and combination of Wasserstein distance based algorithms with task performance such as classification or clustering (see \cite{liu2022wasserstein} for recent work in this direction).}

\section*{Acknowledgements}

KH thanks Alex Cloninger, Longxiu Huang, Anna Little, Daniel McKenzie, James Murphy, Gustavo Rohde, Bernhard Schmitzer, and Matthew Thorpe for helpful discussions regarding this work. The authors thank the referees for their feedback which significantly improved the presentation and results of the paper.

\bibliographystyle{siamplain}
\bibliography{wassmap.bib}

\newpage
\appendix







\section{Python implementation notes}

The Python code for this paper may be found at \url{https://github.com/Wassmap/wassmap}.  For computing Isomap embeddings, we use the \texttt{sklearn.manifold} package, the \texttt{pydiffmap} package for diffusion map embeddings, and \texttt{networkx} for graph-based computations. We use the Python Optimal Transport (POT) package \cite{flamary2021pot} to compute Wasserstein distances between measures. In particular, we use the \texttt{emd2} function to compute Wasserstein distances, which employs the algorithm of \cite{bonneel2011displacement}; we also use the \texttt{sinkhorn} function for entropic regularization of Wasserstein distances. In future work, attention will be paid to using other Wasserstein approximations, but for now we content ourselves with a fixed computational algorithm to illustrate the general results, and do not yet attempt to fully optimize computation time.  

Note that Isomap expects images in pixel/voxel format (i.e., as a 2-d or 3-d array of scalars), whereas Wassmap expects images in point cloud form, i.e., 
\begin{align}
    \mu_{\theta_i} = \sum_{i} f_{\theta_i}(x_i)\delta_{x_i}.
\end{align}
For Wasserstein distance computations, we assume that no $f_{\theta_i}(x_i) =0 $, i.e., points with zero evaluated density are dropped, which may result in different images having distinct number of points. 

In the code, we have provided functions to convert between pixel/voxel and point cloud representations. Typical usage may be found on the GitHub repository.

\section{Including shortest path computations}
\; In the original Isomap paper \cite{tenenbaum2000global} and subsequent work, geodesic distances on the manifold are estimated by first constructing a neighbor graph (typically via $k$-nearest neighbor or $\eps$-neighbor), then employing all-pairs-shortest-path (typically via Dijkstra's or the Floyd-Warshall algorithm) to estimate manifold geodesic distances.  We have proposed a method such that many interesting manifolds do not require this additional step because the 'raw' Wasserstein distance matrices are provably Euclidean (in the sense of MDS).  To illustrate this graphically, consider the translation manifold generated for Figure 1.  Reducing the graph either via $\eps$-neighborhood or $k$NN results in a worse embedding; this is illustrated in \cref{fig:geodesic_test1}.  
\begin{figure}[h!]
    \centering
    \includegraphics[width=0.8\textwidth]{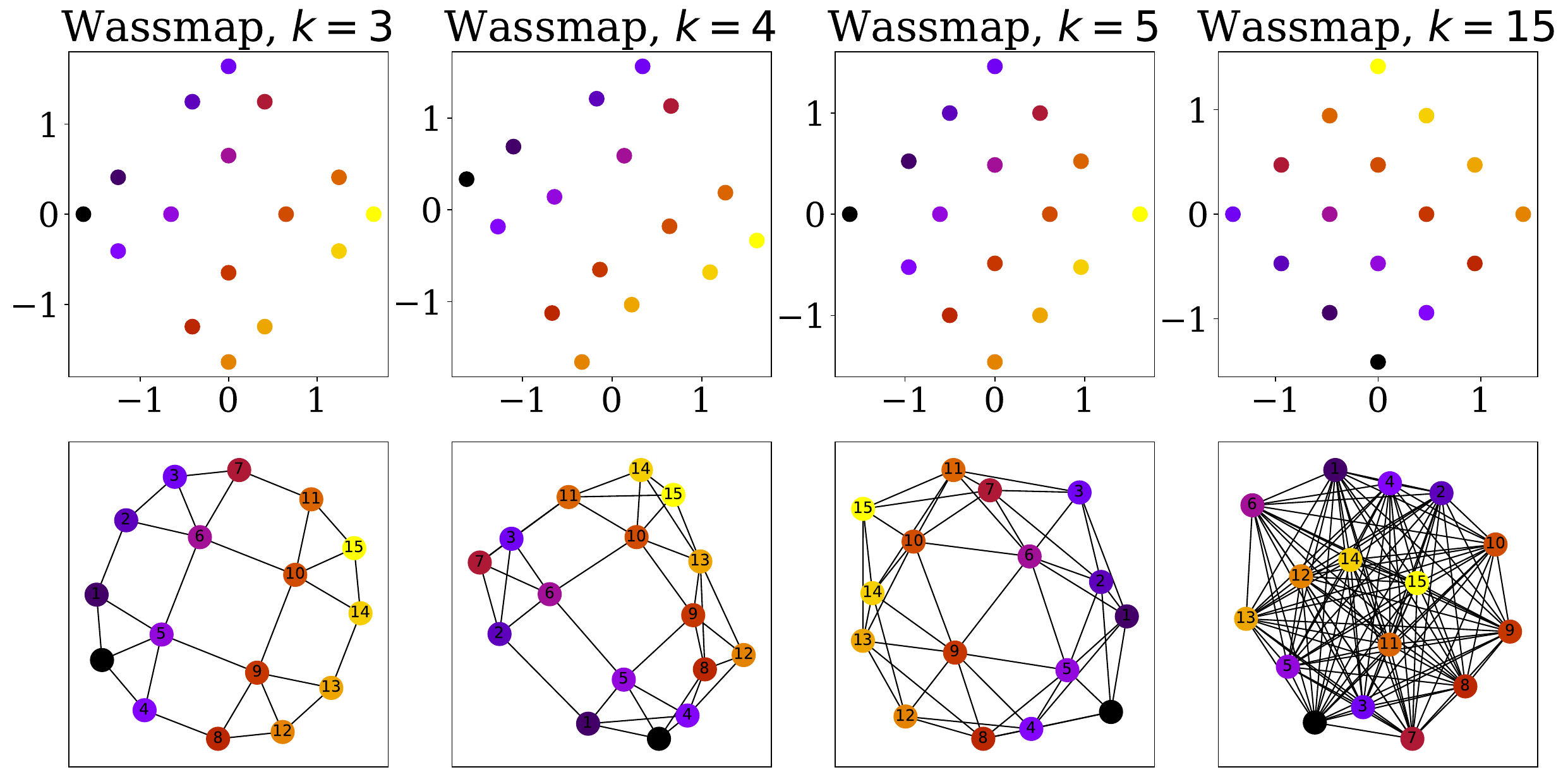}
    \caption{Demonstration of the construction of a Wasserstein $k$NN graph followed by all-pairs-shortest-paths leading to worse embeddings.  The rightmost embedding corresponds to the complete graph, corresponding to the setting of Theorem 3.4. Note that the $k=5$ embedding is slightly curved.}
    \label{fig:geodesic_test1}
\end{figure}

This experiment shows what is already well-understood regarding Isomap and MDS: flat manifolds can be recovered via MDS embeddings, while curved manifolds require the graph geodesic step in Isomap.

\section{Second rotation experiment}

Here we show a second rotation experiment where we take $\mu_0$ to be the indicator function of an ellipse with major radius 1 and minor radius 0.5, but whose initial center is at the origin $(x,y) = (0,0)$. We rotate the ellipse by $N=21$ uniformly sampled angles in $[0,2\pi)$, and the results of the embeddings are shown in \cref{FIG:RotationUncentered}.  Due to rotational symmetry of the base figure, we would not expect any method to recover the `correct' manifold in this case. 

\begin{figure}[h!]
\centering 
     \includegraphics[width=\textwidth]{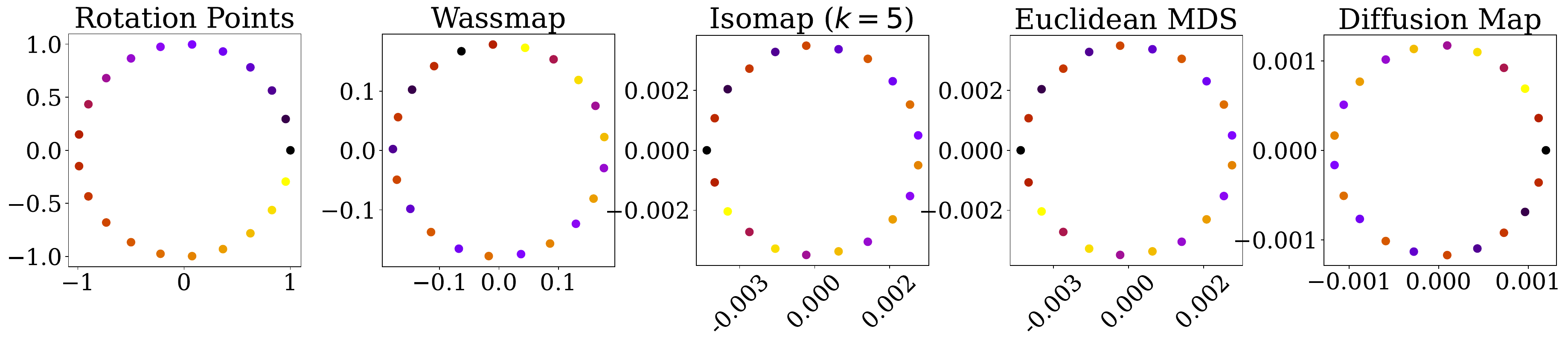}
    \caption{Rotation manifold generated by the characteristic function of an ellipse with major radius 1 and minor radius 0.5 centered at $(0,0)$. Rotation angles ($N=21$) are uniformly sampled between $0$ and $2\pi$. Shown are the original points on the circle $(\cos\theta_i,\sin\theta_i)$, the images $\mu_{\theta_i}$ plotted on the same figure, and the Wassmap embedding.  Note that all methods recover a circle but are unable to maintain the correct order, due to rotational symmetry of the base figure. }\label{FIG:RotationUncentered}
\end{figure}

\section{On low-rankness of Wasserstein distance matrices}

In section 6.2, we discuss the use of the Nystr\"{o}m method for approximating the Wasserstein distance matrix $W$ of the Wassmap algorithm. Use of the Nystr\"{o}m method requires that the kernel matrix in question be approximately low-rank.  Here we show scree plots the spectrum of $W$ from various experiments in the main paper.

\begin{figure}[h!]
    \centering
    \includegraphics[width=.9\textwidth]{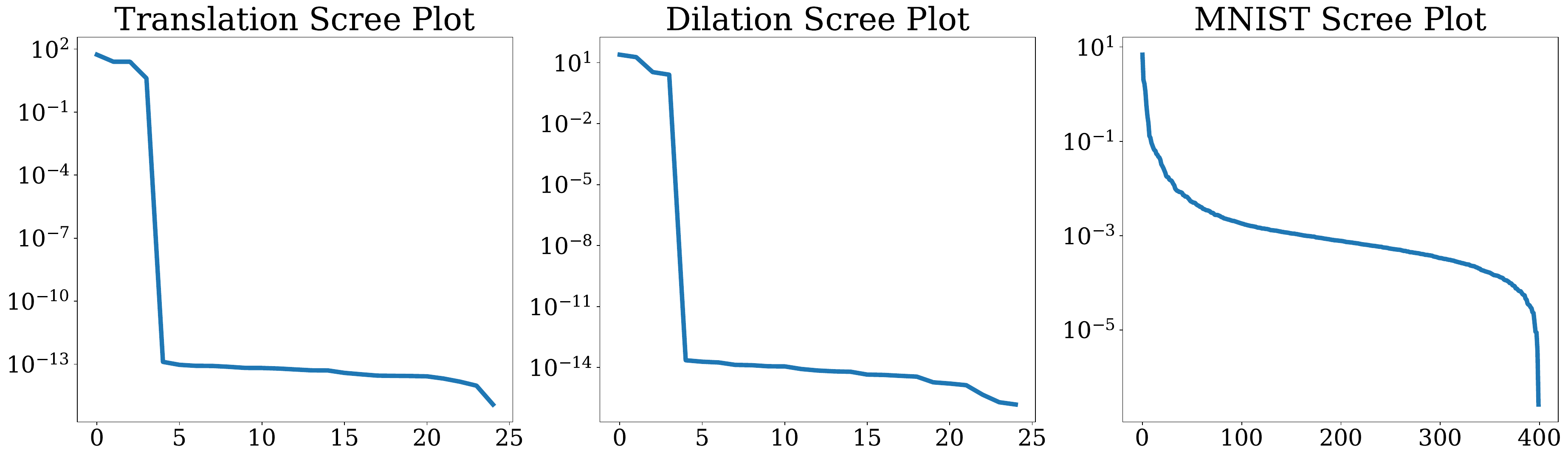}
    \caption{Scree plots of singular values of the square Wasserstein distance matrices for the first translation experiment (left), the dilation experiment (center) and the MNIST experiment (right).}
    \label{FIG:screeplots}
\end{figure}

One can see that the translation distance matrix is well-approximated by a low-rank matrix. The MNIST distance matrix can be approximated by an approximately rank 25 matrix. The dilation matrix has relatively flat spectrum after the initial drop; because of the scale of the singular values ($O(1)$), this matrix is somewhat harder to approximate, but one can see that one gains no approximation power from taking more than about $10$ singular values until one takes almost all of them.

Note that these scree plots can also be used to estimate the dimensionality of the embedding as discussed in section 6.4.

\section{Using Wasserstein distances in other algorithms}

As mentioned above, a na\"{i}ve extension of our idea to algorithms like Laplacian eigenmaps or Diffusion Maps may not be the correct generalization of those methods. For illustration, we repeat the dilation experiment here and in the bottom center use a simple variant of the Diffusion map in which $W_2$ distances are used instead of Euclidean.  One can see in \cref{FIG:LaplaceWass} that the embedding does not do any better than the diffusion map with Euclidean distances, and we attribute this to the fact that a more sophisticated understanding of how to approximate submanifolds of Wasserstein space via spectral graph theory may need to be developed.

\begin{figure}[h!]
    \centering
    \includegraphics[width=.85\textwidth]{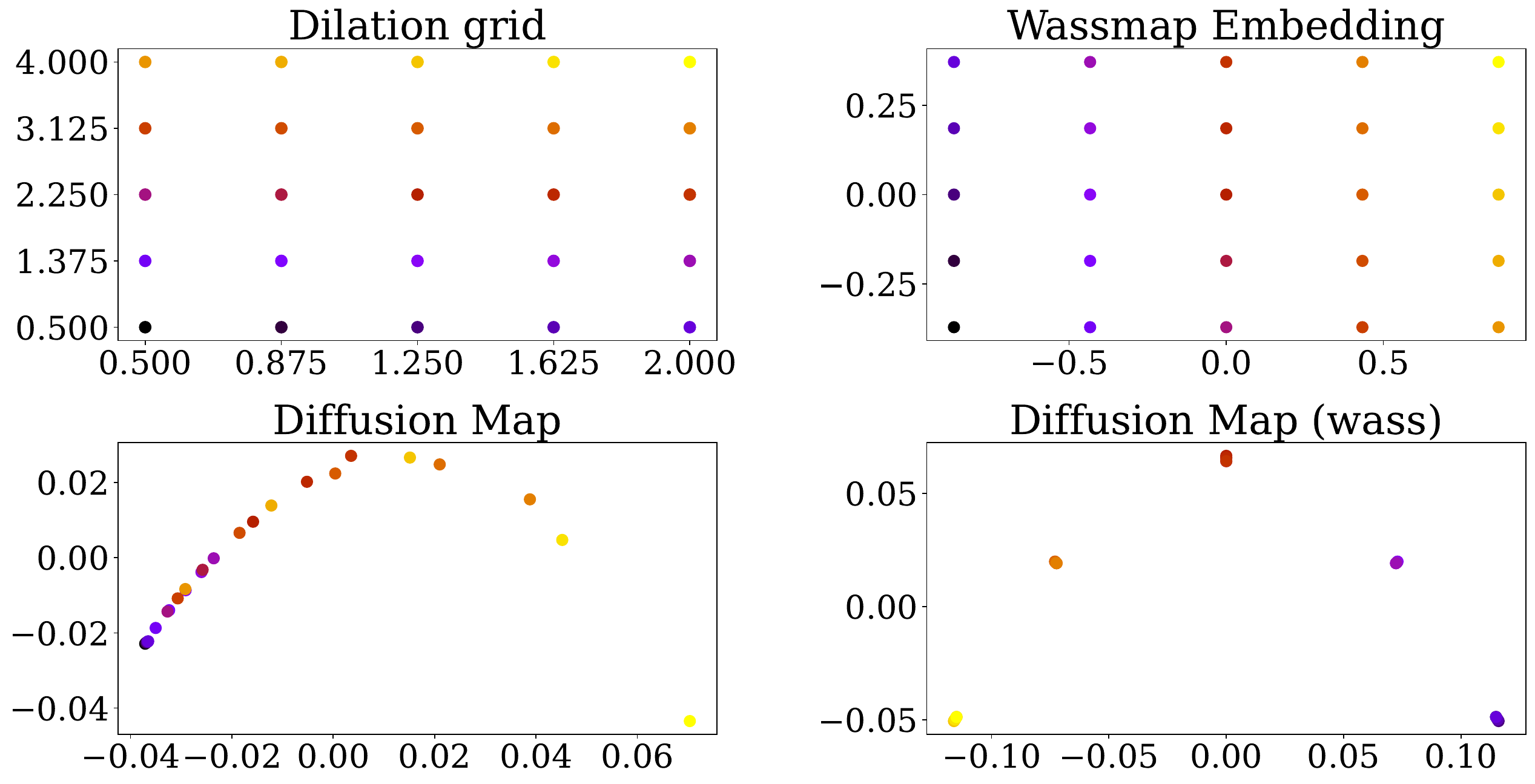}
    \caption{Dilation manifold generated by the characteristic function of the unit disk with parameter set $\Theta=[0.5,2]\times[0.5,4]$. We consider a uniform $5\times 5$ grid to generate $\{\theta_i\}$. Shown are the original dilation grid, the Wassmap embedding, Diffusion Map embedding (bottom left), and Diffusion Map embedding with Wasserstein distance (bottom right).}
    \label{FIG:LaplaceWass}
\end{figure}

\section{Using $W_1$}

We have mainly illustrated the Wassmap algorithm for quadratic Wasserstein distances. However, since $W_1$ distances are also commonly used, it makes sense to consider these. Here, we repeat the dilation example for Wassmap using $W_1$ distances and we see in \cref{FIG:W1} that the embedding fails to recover the grid exactly up to scaling; however, the embedding does still maintain the general characteristics of the dilation set. It is perhaps unsurprising that the qualitative behavior of the $W_1$ distance based embedding is similar to that using $W_2$ distances. Indeed, all $W_p$ distances are based on smoothly distorting one image to the other.  This experiment incidates that further study of the case $p=1$ is warranted.

\begin{figure}[h!]
    \centering
    \includegraphics[width=.8\textwidth]{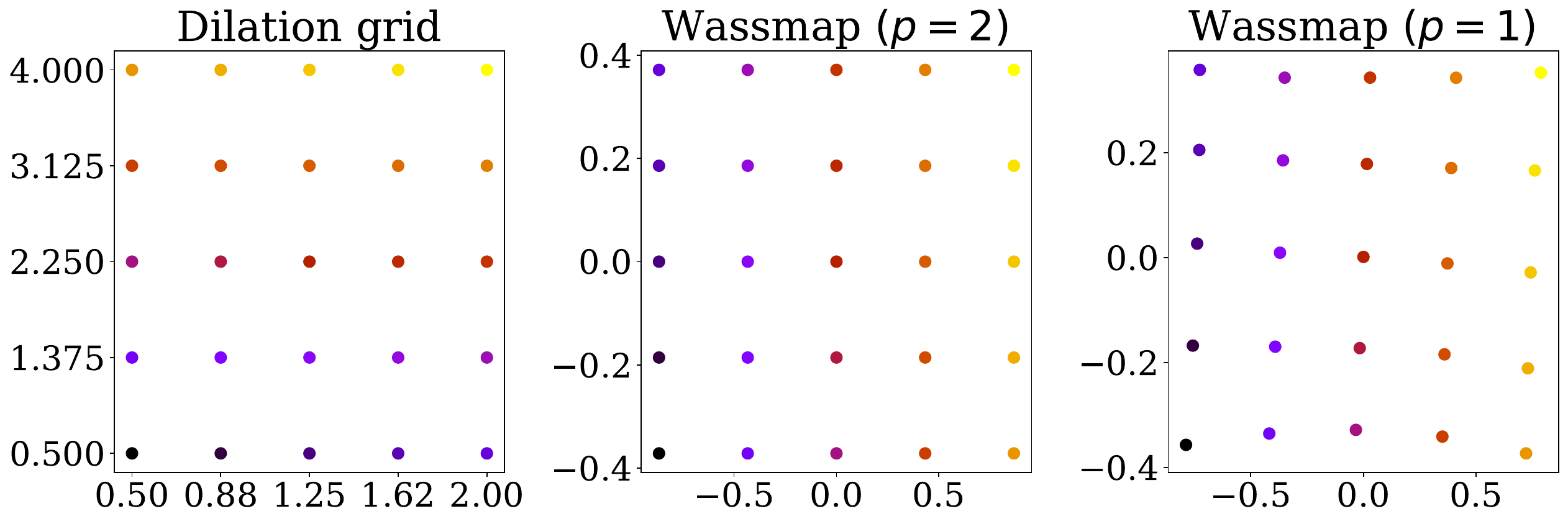} \caption{Original dilation set (left), Wassmap embedding using $W_2$ distances (middle), and Wassmap embedding using $W_1$ distances (right) for the dilation experiment in section 5.2.}
    \label{FIG:W1}
\end{figure}

\section{Mimicking MNIST digits via diffeomorphisms}\label{sec:mnistdiffeo}

Here we will investigate a more general family of diffeomorphisms than those done previously. Our motivation is to mimic the MNIST handwritten digit dataset \cite{lecun1998mnist} by generating morphed elliptic annuli.    We thus consider a parameterized family $f_\theta(x)$ where $\theta = (\theta_0,\theta_1,\theta_2,\theta_3,\theta_4)$ as follows.  First, an elliptic annulus of height one, width $\theta_0$, and thickness $\theta_1$ is generated; denote this by $f^0_\theta(x)$.  Then, a global sheer and local rotation is performed as follows.   Define
\begin{align}
    T_\theta(x) = \left[\begin{array}{cc}
         \cos(\alpha_\theta(x)) & -\sin(\alpha_\theta(x))  \\
         \sin(\alpha_\theta(x)) & \cos(\alpha_\theta(x)) 
    \end{array}\right]x,
\end{align} where $\alpha_\theta(x_1,x_2) = \theta_2\cos(x_1+\theta_3 x_2)\cos(x_2)$, then let $f^1_\theta(x) = f^0_\theta(T_\theta(x))$.  Finally, a global rotation with angle $\theta_4$ is performed, resulting in $f_\theta(x) = f^1_\theta(R_\theta(x))$. We then sample a set of $1296 = 6^4$ such zeros as follows: 
\begin{align*}
    \theta_0 &\sim \text{Unif}(0.2,0.8)\\
    \theta_1 &\sim \text{Unif}(0.05,0.06)\\
    \theta_2 &\sim \text{Unif}(0.2,0.21)\\
    \theta_3 &\sim \text{Unif}(0,0.01)\\
    \theta_4 &\sim \text{Unif}(-\pi/6,\pi/6)
\end{align*}where $\text{Unif}(a,b)$ is the uniform distribution on $(a,b)$.  A subset of 64 such zeros in shown in Figure \ref{fig:deformation_family}.

\Cref{fig:diffeo_2d_3d_comparison} shows the two-dimensional Wassmap and Isomap embeddings of $\{f_{\theta_i}\}$ as defined above, along with the scree plots for each.  The Wassmap embedding shows a more distinctly clustered embedding.  Isomap demonstrates similar structures but does not seem able to separate subtle morphological variations as easily.  The scree plots demonstrate the improved embedding efficiency of Wassmap versus Isomap.  This evidence suggests that Wassmap may be capable of finding the structure of manifolds generated by a restricted family of diffeomorphisms, but future exploration is needed in this case. 

\begin{figure}[h!]
    \centering
        \includegraphics[width=.9\textwidth]{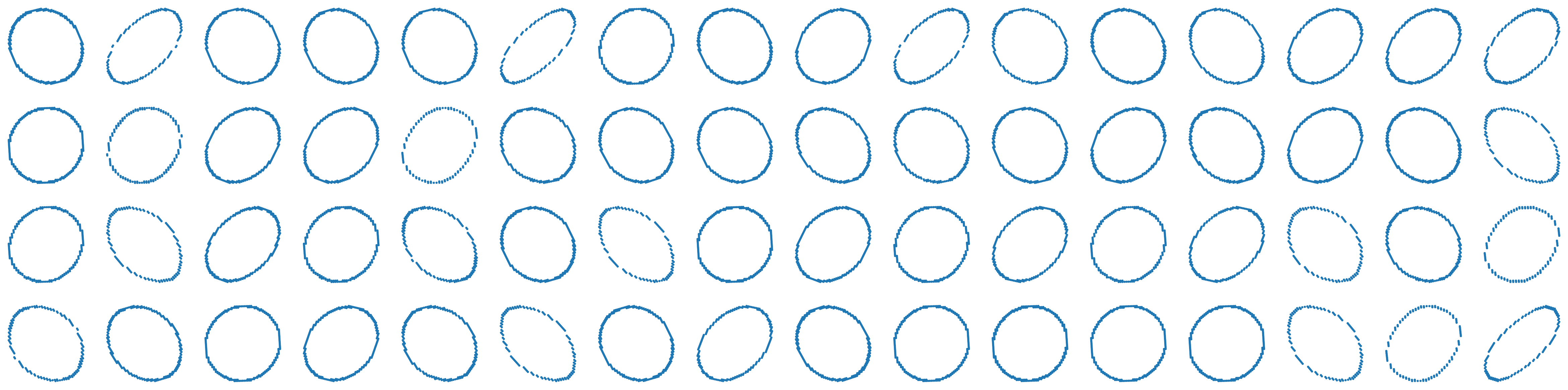}
    \caption{Example simulated MNIST zeros, created by applying transformations to a base elliptical annulus.  }
    \label{fig:deformation_family}
\end{figure}

\begin{figure*}[h!]
    \centering
    \includegraphics[width=.8\textwidth]{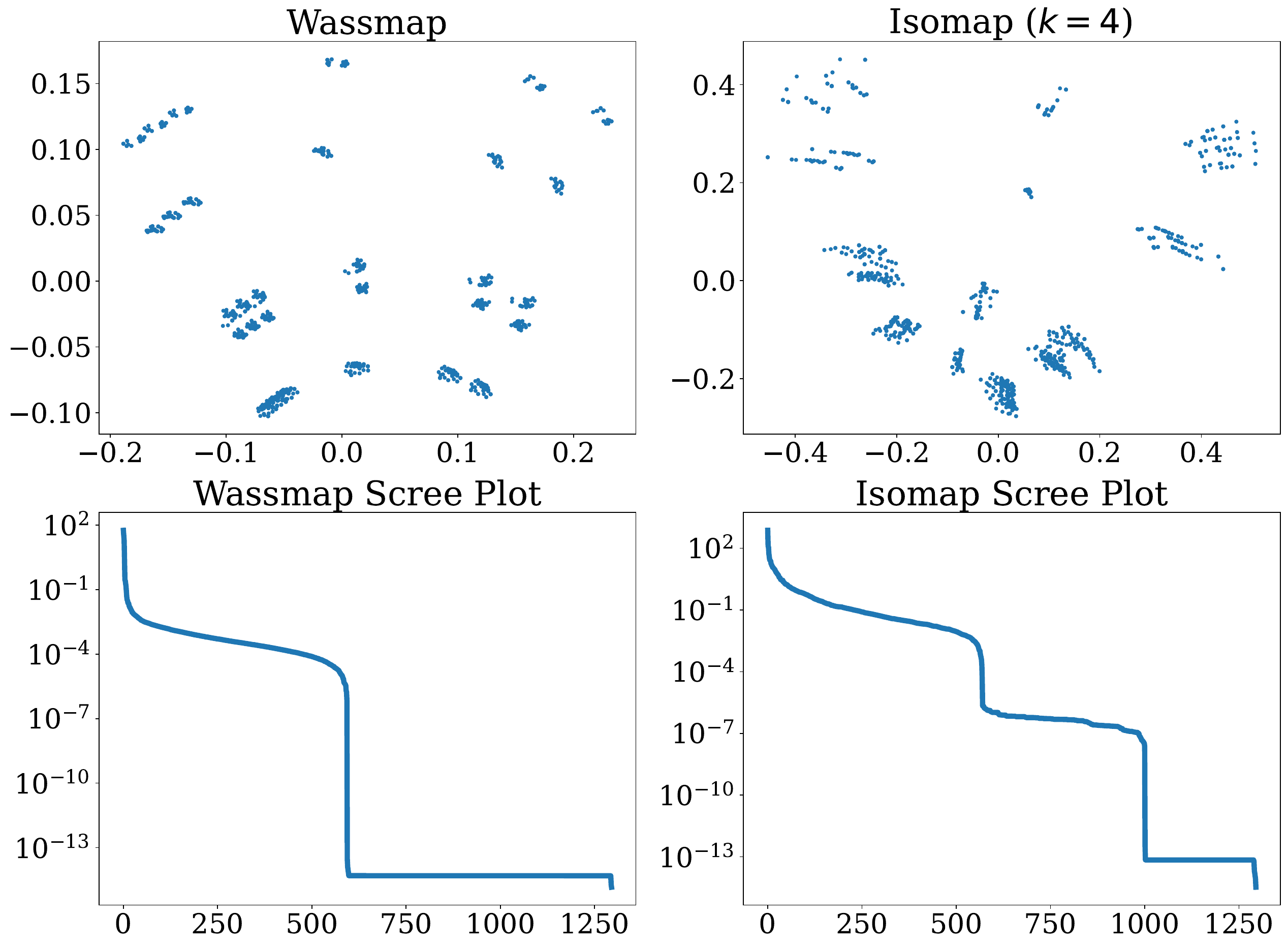}
    \caption{Comparison of the two-dimensional embeddings and scree plots for Wassmap and Isomap for the grid deformation family.  The Wassmap technique produces a much smoother embedding with clear geometric structure, while the Isomap method produces a less coherent embedding with diminshed cluster separation.  The scree plots show that Wassmap is more dimensionally efficient. }
    \label{fig:diffeo_2d_3d_comparison}
\end{figure*}



\end{document}